\documentclass{article}

\usepackage[utf8]{inputenc} 
\usepackage[T1]{fontenc}    
\usepackage{hyperref}       
\usepackage{url}            
\usepackage{booktabs}       
\usepackage{amsfonts}       
\usepackage{nicefrac}       
\usepackage{microtype}      

\newcommand{\argmin}{\operatornamewithlimits{arg\,min}}

\usepackage{bm}
\usepackage{amsmath}
\usepackage{amsthm}
\usepackage{amssymb,amsfonts}

\newtheorem{definition}{Definition}
\newtheorem{remark}{Remark}
\newtheorem{theorem}{Theorem}
\newtheorem{lemma}{Lemma}

\newtheorem{proposition}{Proposition}
\newtheorem{corollary}{Corollary}

\title{Nearly Optimal Clustering Risk Bounds for \\Kernel K-Means}
\author{
Yong Liu$^1$\thanks{Corresponding author}, Lizhong Ding$^2$, Weiping Wang$^{1}$\\
$ ^1$Institute of Information Engineering, Chinese Academy of Sciences\\
$ ^2$Inception Institute of Artificial Intelligence\\
\texttt{liuyong@iie.ac.cn}\\
}

\begin{document}

\maketitle

\begin{abstract}
    In this paper, we study the statistical properties of kernel $k$-means
    and obtain a nearly optimal excess clustering risk bound, substantially improving the state-of-art bounds
    in the existing clustering risk analyses.
    We further analyze the statistical effect of computational approximations of the Nystr\"{o}m kernel $k$-means,
    and prove that it achieves the same statistical accuracy as the exact kernel $k$-means considering only $\Omega(\sqrt{nk})$
    Nystr\"{o}m landmark points.
    To the best of our knowledge, such sharp excess clustering risk bounds for kernel (or approximate kernel) $k$-means have never been proposed before.
\end{abstract}

\section{Introduction}
Clustering, a fundamental data mining task, has found use in
a variety of applications such as web search, social network
analysis, image retrieval, medical imaging, gene expression
analysis, recommendation systems and market analysis \cite{jain2010data}.
$k$-means is arguably one of the most common approaches to clustering,
producing clusters with piece-wise linear boundaries.
Its kernel version, which employs a nonlinear distance function, has the ability to find clusters of
varying densities and distributions, characteristics inherent in many real datasets,
greatly improving the flexibility of the approach \cite{dhillon2004kernel,wang2019scalable}.

To understand the (kernel) $k$-means and guide the development of new clustering algorithms,
many researchers  have investigated its theoretical properties for decades.
The consistency of the empirical minimizer
was demonstated by \cite{pollard1981strong,pollard1982quantization,abaya1984convergence}.
Rates of convergence and non-asymptotic performance bounds were considered
by \cite{pollard1982central,chou1994distortion,linder1994rates,bartlett1998minimax,linder2000training}.
Most of the proposed risk bounds are dependent upon the dimension of the hypothesis space.
For example, \cite{bartlett1998minimax} provided, under some mild assumptions,
a clustering risk bound of order $\mathcal{O}(\sqrt{kd/n})$, where $d$ is the dimension of the hypothesis space.
However,  the hypothesis space of kernel $k$-means is typically an infinite-dimensional Hilbert space,
such as the reproducing kernel Hilbert space (RKHS) associated with Gaussian kernels \cite{Scholkopf2002Learning}.
Thus, the existing theoretical analyses of $k$-means are usually not suitable for analyzing its kernel version.
In recent years,
\cite{biau2008performance,canas2012learning,maurer2010k,antos2005individual,levrard2015nonasymptotic,koltchinskii2006local,calandriello2018statistical} extended the previous results,
and provided dimension-independent bounds for kernel $k$-means.
Moreover, as shown in \cite{biau2008performance},
if the feature map associated with the kernel function satisfies$\|\Phi\|\leq 1$,
then the clustering risk bounds are of order $\mathcal{O}(k/\sqrt{n})$.
These clustering risk bounds  for kernel $k$-means are usually linearly dependent on the number of clusters $k$.
However, for the fine-grained analysis in social network or recommendation systems,
the clusters $k$  may be very large.
Thus, from the theoretical perspective,
these existing bounds of $\mathcal{O}(k/\sqrt{n})$ do not match the
 stated lower bound $\Omega(\sqrt{k/n})$ in $k$ \cite{bartlett1998minimax}.
Whether or not it is possible to
prove a optimal risk bound of rate $\mathcal{O}(\sqrt{k/n})$ in both $k$ and $n$
 is still an open question \cite{biau2008performance,calandriello2018statistical}.

Although kernel $k$-means is one of the most popular clustering methods,
it requires the computation of an $n\times n$ kernel matrix.
As for other kernel methods, this becomes unfeasible for large-scale
problems and thus deriving approximate computations has been the subject to numerous recent works,
such as partial decompositions \cite{bach2005predictive},
random projection \cite{biau2008performance,cohen2015dimensionality},
Nystr\"{o}m approximations based on uniform sampling \cite{drineas2005nystrom,chitta2011approximate,calandriello2018statistical,oglic2017nystrom,wang2019scalable},
and random feature approximations \cite{Rahimi2007Random,Chitta2012,Bach2015Equivalence,rudi2017generalization}.
However, very few of these optimization-based methods focus on the underlying excess risk problem.
To the best of our knowledge,
the only two results providing excess risk guarantees for approximate kernel $k$-means are
\cite{biau2008performance} and \cite{calandriello2018statistical}.
In \cite{biau2008performance}, the authors considered the excess risk of the empirical risk minimization (ERM)
when the approximate Hilbert space is obtained using Gaussian projections.
In \cite{calandriello2018statistical},
they showed that, when sampling $\Omega(\sqrt{n})$ Nystr\"{o}m landmarks,
the excess risk bound can reach $\mathcal{O}(k/\sqrt{n})$.
The excess risk bounds of \cite{calandriello2018statistical} and \cite{biau2008performance}
are both linearly dependent on $k$
and thus do not match the theoretical lower bound \cite{bartlett1998minimax}.

In this paper,
we study the kernel $k$-means in terms of both statistical and computational requirements.
The main content of this paper is divided into two parts.
In the first part,
we study the statistical properties of kernel $k$-means
and obtain an excess clustering risk bound with a convergence rate of $\tilde{\mathcal{O}}(\sqrt{k/n})$\footnote{$\tilde{\mathcal{O}}$ hides logarithmic terms.},
which is nearly optimal in both $k$ and $n$.
In the second part,
we quantify the statistical effect of computational approximations of the Nystr\"{o}m-based kernel $k$-means,
and prove that sampling $\Omega(\sqrt{nk})$ (or $\Omega(\sqrt{n})$ under a certain basic assumption)
Nystr\"{o}m landmarks allows us to greatly reduce the computational costs without incurring in asymptotic loss of accuracy.
The major contributions of this paper include:
\begin{itemize}
  \item Inspired by \cite{foster2019ell},
  a sharp bound of clustering Rademacher complexity for kernel $k$-means is provided (see Theorem \ref{the-main-restult}),
   which is linearly dependent on $\sqrt{k}$, substantially improving the existing bounds.
  \item Based on the sharp bound clustering Rademacher complexity, a nearly optimal excess clustering risk bound of rate $\tilde{\mathcal{O}}(\sqrt{k/n})$ for empirical risk minimizer (ERM) is proposed (see Theorem \ref{the-three}).
  To the best of our knowledge, this is the first (nearly) optimal excess risk bound for kernel $k$-means in both $k$ and $n$.
  Beyond ERM, we further extend the result of Theorem \ref{the-three} for general cases (see Corollary \ref{the-three-approximate} and Corollary \ref{cor-gajgagh}).
  \item A (nearly) optimal excess risk bound for Nystr\"{o}m kernel $k$-means is also obtained
  when sampling $\Omega(\sqrt{nk})$ points (see Theorem \ref{the-nsty}) or $\Omega(\sqrt{n})$ under a certain basic assumption (see Corollary \ref{cor-first} and Corollary \ref{cor-second}).
  This result shows that we can improve the computational aspect of kernel $k$-means using Nystr\"{o}m embedding,
  while maintaining optimal generalization guarantees.
\end{itemize}

The rest of the paper is organized as follows.
In Section 2, we introduce some notations and provide an overview of kernel $k$-means.
In Section 3, we first derive an upper bound of the clustering Rademacher complexity,
and further provide nearly optimal excess risk bounds.
In Section 4, we  quantify the statistical effect of computational approximations of the Nystr\"{o}m-based kernel $k$-means.
We end in Section 5 with a conclusion.
All the proofs are given in the supplementary materials.

\section{Background}
In this section, we will give some notations and provide a brief introduction to kernel $k$-means,
please refer to \cite{dhillon2004kernel,calandriello2018statistical} for details.
\subsection{Notation}
Assume $\mathbb{P}$ is a (unknown) distribution on $\mathcal{X}$,
and $\mathcal{S}=\left\{\mathbf  x_i\right\}_{i=1}^n\in\mathcal{X}$ are $n$ samples
drawn i.i.d from $\mathbb{P}$.
We denote with $\mathbb{P}_n(\mathcal{S})=1/n\sum_{i=1}^n \mathbf 1\{\mathbf x_i\in \mathcal{S}\}$ the $empirical$ distribution.
Let $\kappa: \mathcal{X}\times\mathcal{X}\rightarrow \mathbb{R}$ be a mercer kernel,
and $\mathcal{H}$ be its  associated reproducing kernel Hilbert space (RKHS)
which is the completion of the linear span of the set of functions: $\mathcal{H}=\overline{\mathrm{span}\{\kappa(\mathbf x,\cdot),\mathbf x\in\mathcal{X}\}}$.
Denote by $\mathcal{H}^k=\otimes_{i=1}^k\mathcal{H}$ the Cartesian product of $\mathcal{H}$.
We use the $feature$ $map$ $\psi:\mathcal{X}\rightarrow \mathcal{H}$ to map $\mathcal{X}$ into the Hilbert space $\mathcal{H}$,
and assume that $\mathcal{H}$ is separable,
such that for any $\mathbf x\in\mathcal{X}$ we have $\Phi_\mathbf{x}=\psi(\mathbf x)$.
Intuitively, in the rest of the paper, the reader can assume that $\Phi_\mathbf{x}\in \mathbb{R}^d$ with $d\gg n$ or even infinite.
Throughout, we will denote by $\langle \cdot,\cdot\rangle$ the inner product in $\mathcal{H}$,
and by $\|\cdot\|$ the associated norm,
and assume that $\|\Phi_\mathbf{x}\|\leq 1$ for any $\mathbf x\in\mathcal{X}$.
We denote with $\mathcal{D}=\left\{\Phi_i=\psi(\mathbf x_i)\right\}_{i=1}^n$ the input dataset, and
 $[\mathbf K]_{i,j}=\kappa(\mathbf x_i,\mathbf x_j)=\langle \Phi_i, \Phi_j\rangle$ the kernel matrix.

\subsection{Kernel $k$-Means}
In this paper, we aim at partitioning the given dataset into $k$ disjoint $clusters$, each characterized
 by its $centroid$ $\mathbf c_j$.
The Voronoi cell associated with a centroid $\mathbf c_j$ is defined as set
$
  \mathcal{C}_j:=\left\{i: j=\argmin_{s=1,\ldots,k}\|\Phi_i-\mathbf c_s\|^2\right\}.
$
Let $\mathbf{C}=[\mathbf c_1,\ldots,\mathbf c_k]$ be a collection of $k$ centroids from $\mathcal{H}^k$.
In this paper, we focus on the so-called $kernel$ $k$-$means$ clustering,
by minimizing the $empirical~squared~norm~criterion$
\begin{align}
  \label{emp-sq-cri}
  \mathcal{W}(\mathbf C,\mathbb{P}_n):=\frac{1}{n}\sum_{i=1}^n \min_{j=1,\ldots,k}\|\Phi_i-\mathbf c_j\|^2
\end{align}
over all possible choices of cluster centers $\mathbf C\in\mathcal{H}^k$.
From \cite{dhillon2004kernel,calandriello2018statistical}, we know that the $\mathcal{W}(\mathbf C,\mathbb{P}_n)$ can be written as
$\frac{1}{n}\sum_{j=1}^k\sum_{i\in \mathcal{C}_j}\left\|\Phi_i -\frac{1}{|\mathcal{C}_j|}\sum_{t\in \mathcal{C}_j}\Phi_t\right\|^2$
with $\left\|\Phi_i -\frac{1}{|\mathcal{C}_j|}\sum_{t\in \mathcal{C}_j}\Phi_t\right\|^2
=\kappa(\mathbf x_i,\mathbf x_i)-\frac{2}{|\mathcal{C}_j|}\sum_{t\in \mathcal{C}_j}\kappa(\mathbf x_i,\mathbf x_t)
+\frac{1}{|\mathcal{C}_j|^2}\sum_{t,t'\in\mathcal{C}_j}\kappa(\mathbf x_t,\mathbf x_{t'})$.
The $empirical$ $risk$ $minimizer$ $(ERM)$ is defined as
\begin{align*}
  \mathbf{C}_n:=\argmin_{\mathbf C\in\mathcal{H}^k}\mathcal{W}(\mathbf C,\mathbb{P}_n).
\end{align*}
The performance of a clustering scheme given by the collection $\mathbf{C}\in \mathcal{H}^k$ of cluster centers
is usually measured by the $expected~squared~norm~criterion$ or $expected~clustering~risk$
\begin{align*}
  \mathcal{W}(\mathbf C, \mathbb{P}):=\int \min_{j=1,\ldots,k}\|\Phi_\mathbf{x}-\mathbf c_j\|^2 d\mathbb{P}(\mathbf x).
\end{align*}
Given a $\mathbf C\in\mathcal{H}^k$,
let $f_\mathbf{C}$
be a $vector$-$valued$ function  associated with the collection $\mathbf C$,
$f_\mathbf{C}=(f_{\mathbf c_1},\ldots, f_{\mathbf c_k}), f_{\mathbf c_j}(\mathbf x)=\|\Phi_\mathbf{x}-\mathbf c_j\|^2$
and
$\mathcal{F}_\mathbf{C}$ be a family of $vector$-$valued$ functions with
\begin{align}
\label{def-FC}
    \begin{aligned}
      \mathcal{F}_\mathbf{C}:=&\Big\{ f_\mathbf{C}=(f_{\mathbf c_1},\ldots, f_{\mathbf c_k}):
      f_{\mathbf c_j}(\mathbf x)=\|\Phi_\mathbf{x}-\mathbf c_j\|^2, \mathbf C\in\mathcal{H}^k\Big\}.
    \end{aligned}
\end{align}
Let $\varphi:\mathbb{R}^k\rightarrow \mathbb{R}$ be a minimum function:
$\varphi(\bm \alpha)=\min_{i=1,\ldots,k}\alpha_i,$
and $\mathcal{G}_\mathbf{C}$ be a "minimum" family of the functions of $\mathcal{F}_\mathbf{C}$,
\begin{align}
  \label{def-G}
  \mathcal{G}_\mathbf{C}:=\left\{g_\mathbf{C}=\varphi\circ f_\mathbf{C}: g_\mathbf{C}(\mathbf x)=\varphi(f_\mathbf{C}(\mathbf x))=\min(f_{\mathbf c_1}(\mathbf x),\ldots, f_{\mathbf c_k}(\mathbf x)),
  f_\mathbf{C}\in \mathcal{F}_\mathbf{C}\right\}.
\end{align}
From the definition of $\varphi(f_\mathbf{C}(\mathbf x))=\min(f_{\mathbf c_1}(\mathbf x),\ldots, f_{\mathbf c_k}(\mathbf x)),$
one can see that the empirical and expected error $\mathcal{W}(\mathbf C,\mathbb{P}_n)$ and $\mathcal{W}(\mathbf C, \mathbb{P})$ can be respectively written as
\begin{align*}
  \mathcal{W}(\mathbf C,\mathbb{P}_n):=\frac{1}{n}\sum_{i=1}^n\varphi(f_\mathbf{C}(\mathbf x_i)) \text{ and }
  \mathcal{W}(\mathbf C,\mathbb{P}):=\int \varphi(f_\mathbf{C}(\mathbf x)) d\mathbb{P}(\mathbf x).
\end{align*}
In this paper, we consider bounding the $excess$ $clustering$ $risk$ $\mathcal{E}(\mathbf C_n)$ of the empirical risk minimizer:
\begin{align*}
  \mathcal{E}(\mathbf C_n)=\mathbb{E}_{\mathcal{D}}[\mathcal{W}(\mathbf C_n,\mathbb{P})]-\mathcal{W}^\ast(\mathbb{P}),
\end{align*}
where $\mathcal{W}^\ast(\mathbb{P})=\inf_{\mathbf C\in\mathcal{H}^k}\mathcal{W}(\mathbf C,\mathbb{P})$ is the optimal clustering risk.
\subsection{The Existing Excess Clustering Risk Bounds for Kernel $k$-Means}
According to \cite{bartlett1998minimax}, we know that
there exists a collection of centroids $\mathbf C\in\mathcal{H}^k$, a constant $c$,
and a distribution $\mathbb{P}$ with $\|\Phi_\mathbf{x}\|\leq 1$ for any $\mathbf x\in\mathcal{X}$,
such that
\begin{align*}
  \mathbb{E}_\mathcal{D}[\mathcal{W}(\mathbf C, \mathbb{P})]-\mathcal{W}^\ast(\mathbb{P}) \geq c\sqrt{\frac{k^{1-4/d}}{n}}.
\end{align*}
Note that $d$ is the dimension of $\Phi_\mathbf{x}$,
which is usually very large or  even infinite.
Thus, the lower bound of kernel $k$-means should be $\Omega\big(\sqrt{{k}/{n}}\big)$.
However, most of the existing risk bounds proposed for kernel $k$-means are $\mathcal{O}(k/\sqrt{n})$
\cite{biau2008performance,canas2012learning,maurer2010k,calandriello2018statistical}:
\begin{theorem}[\cite{biau2008performance}, Theorem 2.1]
  If $\|\Phi_\mathbf{x}\|\leq 1$ for any $\mathbf x\in\mathcal{X}$, then there exists a constant $c$ such that
  \begin{align*}
    \mathbb{E}_\mathcal{D}[\mathcal{W}(\mathbf C_n, \mathbb{P})]-\mathcal{W}^\ast(\mathbb{P})\leq c\frac{k}{\sqrt{n}}.
  \end{align*}
\end{theorem}
The linear dependence on $k$ is a consequence of the proof technique of \cite{biau2008performance}.
Note that, the clusters $k$ may be very large for the
fine-grained analysis in social network or recommendation systems.
Whether or not it is possible to
prove a bound of rate $\sqrt{k/n}$,
which is (nearly) optimal in both $k$ and $n$,
is still an open question \cite{biau2008performance,calandriello2018statistical}.
In this paper, we attempt to fill this gap.
\section{Main Results}
Our discussion on excess risk bounds is based on the established methodology
of Rademacher complexity \cite{Bartlett2002ff}.
To this end, we first introduce the notion of clustering Rademacher complexity.
\begin{definition}[Clustering Rademacher Complexity \cite{biau2008performance}]
  Let $\mathcal{G}_\mathbf{C}$ be a family of functions defined in \eqref{def-G},
    $\mathcal{S}=(\mathbf x_1,\ldots,\mathbf x_n)$ a fixed sample of size $n$ with elements in $\mathcal{X}$,
   and $\mathcal{D}=\left\{\Phi_i=\psi(\mathbf x_i)\right\}_{i=1}^n$ the input dataset.
  Then, the clustering empirical Rademacher complexity of $\mathcal{G}_\mathbf{C}$ with respect to  $\mathcal{D}$ is defined by
  \begin{align*}
    {\mathcal{R}_n}(\mathcal{G}_\mathbf C)=\mathbb{E}_{\bm \sigma}\left[\sup_{g_\mathbf{C}\in\mathcal{G}_\mathbf{C}}
    \sum_{i=1}^n\sigma_i g_\mathbf{C}(\mathbf x_i)\right],
  \end{align*}
  where $\sigma_1,\ldots,\sigma_n$ are independent random variables with equal probability taking values $+1$ or $-1$.
  Its expectation is
  $\mathcal{R}(\mathcal{G}_\mathbf{C})= \mathbb{E}_\mathcal{D}\left[{\mathcal{R}_n}(\mathcal{G}_\mathbf{C})\right].$
\end{definition}

Existing work on data-dependent generalization bounds for $k$-means \cite{biau2008performance,canas2012learning,maurer2010k,koltchinskii2006local,calandriello2018statistical}
usually builds on the following structural result for the clustering Rademacher complexity:
\begin{align}
\label{equ-trrad}
  {\mathcal{R}_n}(\mathcal{G}_\mathbf{C})\leq \sum_{j=1}^k{\mathcal{R}_n}(\mathcal{F}_{\mathbf C_j}),
\end{align}
where
$\mathcal{F}_{\mathbf C_j}=\left\{f_{\mathbf c_j}: f_{\mathbf c_j}\in \mathcal{F}_\mathbf{C}\big|_j\right\}$,
$\mathcal{F}_\mathbf{C}\big|_j$ is the output coordinate $j$ of $\mathcal{F}_\mathbf{C}$.
Thus, the existing risk bounds  of $k$-means determined by the clustering Rademacher complexity are linearly dependent on $k$.
Inspired by the recently work \cite{foster2019ell}, in the following, we provide an improved bound for kernel $k$-means.
\subsection{A Sharp Bound for the Clustering Rademacher Complexity}
\begin{theorem}
  \label{the-main-restult}
  If $\forall \mathbf x\in\mathcal{X}, \|\Phi_\mathbf{x}\|\leq 1$.
  Then,
  for any  $\delta>0$ and $\mathcal{S}=\{\mathbf x_1,\ldots, \mathbf x_n\}\in\mathcal{X}^n$, there exists a constant $c>0$ such that
  \begin{align*}
    &{\mathcal{R}_n}(\mathcal{G}_\mathbf{C})\leq c\sqrt{k}\max_i \tilde{\mathcal{R}}_n(\mathcal{F}_{\mathbf C_i})
    \log ^{\frac{3}{2}+\delta}\Big(\frac{n}{\max_i \tilde{\mathcal{R}}_n(\mathcal{F}_{\mathbf C_i})}\Big),
  \end{align*}
  where $\tilde{\mathcal{R}}_n(\mathcal{F}_{\mathbf C_i})=\max_{\mathcal{S}\in \mathcal{X}^n}\mathcal{R}_n(\mathcal{F}_{\mathbf C_i})$.
\end{theorem}
The above result shows that the upper bound of the clustering Rademacher complexity is linearly dependent on $\sqrt{k}$,
which substantially improves the existing bounds linearly dependent on $k$.

\begin{remark}
  The upper bound of clustering Rademacher complexity involves a constant $c$ and a logarithmic term,
  thus if one requires it's absolute value to be smaller than the existing bounds defined in \eqref{equ-trrad},
  there may exist some cases which acquire a large size of $k$.
  However, from the statistical perspective, our bound of linear dependence on $\sqrt{k}$ substantially improves the existing ones of linear dependence on $k$.
\end{remark}
\begin{remark}
  In  the recently work \cite{foster2019ell}, they show that the Rademacher complexity of the $k$-valued function class of $L$-Lipschitz
  with respect to $L_\infty$ norm can be bounded by the maximum Rademacher complexity of the restriction of the function class along each coordinate, times a factor of $\tilde{\mathcal{O}}(\sqrt{k})$.
  In this paper, we extend the result of \cite{foster2019ell} from the linear space to RKHS.
\end{remark}
\begin{proof}[Proof of sketch of Theorem \ref{the-main-restult}]
  We first show that the function $\varphi(\bm \nu)=\min(\nu_1,\ldots,\nu_k)$
  is 1-Lipschitz with respect to the $L_\infty$-norm, i.e.,
    $
      \forall \bm \nu,\bm \nu'\in\mathbb{R}^k, |\varphi(\bm \nu)-\varphi(\bm \nu')|\leq \|\bm \nu-\bm \nu'\|_\infty.
   $
  Then, inspired by \cite{foster2019ell},
  we prove that clustering Rademacher complexity of 1-Lipschitz function with respect to the $L_\infty$-norm can be bound by
  the maximum Rademacher complexity of each coordinate with a factor of $\sqrt{k}$.
\end{proof}


In the following, we will show that
Theorem \ref{the-main-restult} cannot be improved from statistical view when ignoring the logarithmic terms.
\begin{proposition}
\label{propo-lowbound}
  There exists a set $\mathbf{C}\in\mathcal{H}^k$
  and data sequence $\mathcal{D}=\{\Phi_1,\ldots,\Phi_n\}$ for which
  \begin{align*}
    \mathcal{R}_n(\mathcal{G}_\mathbf{C})\geq \frac{\sqrt{k}}{3\sqrt{2}}\cdot\max_i\tilde{\mathcal{R}}_n(\mathcal{F}_{\mathbf C_i}).
  \end{align*}
\end{proposition}
The above result shows that the lower bound of $\mathcal{R}_n(\mathcal{G}_\mathbf{C})$
is $\Omega\big(\sqrt{k}\max_i\tilde{\mathcal{R}}_n(\mathcal{F}_{\mathbf C_i})\big)$,
which implies that the upper bound $\tilde{\mathcal{O}}\big(\sqrt{k}\max_i\tilde{\mathcal{R}}_n(\mathcal{F}_{\mathbf C_i})\big)$
in Theorem \ref{the-main-restult} is (nearly) optimal.

\begin{proof}[Proof of sketch of Proposition \ref{propo-lowbound}]
  We first prove that $\max_i\tilde{\mathcal{R}}_n(\mathcal{F}_{\mathbf C_i})$ can be bounded by $3\sqrt{n}$.
  Then, we show that there exist a hypothesis class
  \begin{align*}
  \mathcal{F}_\mathbf{C}=\left\{f_\mathbf{C}=(f_{\sigma_1\cdot e_1},\ldots, f_{\sigma_k\cdot e_k}) ~\Big|~f_{\sigma_i\cdot e_i}(\mathbf x)=\|\Phi-\sigma_i\cdot e_i\|^2, \bm \sigma\in\{\pm 1\}^k \right\},
\end{align*}
such that $\mathcal{R}_n(\mathcal{G}_\mathbf{C})\geq \sqrt{\frac{kn}{2}}$,
where $e_i$ is the $i$th standard basis function in $\mathcal{H}$,
  and $\bm \sigma\in\{\pm 1\}^k$ are Rademacher variables.
\end{proof}

\begin{remark}
  In \cite{foster2019ell}, they prove that there exists a $k$-valued function class $\mathcal{F}$, such that
  $\mathcal{R}_n(\mathcal{F})\geq \Omega(k)\max_i\tilde{R}_n(\mathcal{F}_i).$
  Note that this lower bound is linearly dependent on $k$ which does not match the upper bound of $\sqrt{k}$,
  while ours in Proposition \ref{propo-lowbound} matches.
\end{remark}

\subsection{A Sharp Excess Risk Bound for Kernel $k$-Means}
\begin{theorem}
  \label{the-three}
  If $\forall \mathbf x\in\mathcal{X}, \|\Phi_\mathbf{x}\|\leq 1$,
  then for any $\delta>0$, there exists a constant $c$,
  with probability  at least $1-\delta$, we have
  \begin{align*}
    \mathbb{E}[\mathcal{W}(\mathbf C_n,\mathbb{P})]
    -\mathcal{W}^\ast(\mathbb{P})
    \leq  c \sqrt{\frac{k}{n}}\log ^{\frac{3}{2}+\delta}
    \left(\frac{n}{\max_i \tilde{\mathcal{R}}_n(\mathcal{F}_{\mathbf C_i})}\right).
  \end{align*}
\end{theorem}
From the above theorem, we know that
$
 \mathbb{E}[\mathcal{W}(\mathbf C_n,\mathbb{P})]-\mathcal{W}^\ast(\mathbb{P})= \tilde{\mathcal{O}}\big(\sqrt{{k}/{n}}\big),
$
which matches the theoretical lower bound $\Omega\big(\sqrt{{k}/{n}}\big)$ \cite{bartlett1998minimax}  when $d$ is large.
Thus, our proposed bound is (nearly) optimal.

\begin{proof}[Proof of sketch of Theorem \ref{the-three}]
According to  \cite{Bartlett2002ff,biau2008performance},
we first show that
\begin{align*}
  &\mathbb{E}[\mathcal{W}(\mathbf C_n,\mathbb{P})]-\mathcal{W}^\ast(\mathbb{P})
  =\mathbb{E}\Big[\big(\mathcal{W}(\mathbf C_n,\mathbb{P})-\mathcal{W}(\mathbf C_n,\mathbb{P}_n)\big)+\big(\mathcal{W}(\mathbf C_n,\mathbb{P}_n))-\mathcal{W}^\ast(\mathbb{P}\big)\Big]\\
  \leq& 2\mathbb{E}\sup_{\mathbf C\in\mathcal{H}^k}\left(\mathcal{W}(\mathbf C,\mathbb{P}_n)-\mathcal{W}(\mathbf C,\mathbb{P})\right)\leq \frac{4}{n}\mathcal{R}(\mathcal{G}_\mathbf{C})\leq\frac{4\mathcal{R}_n(\mathcal{G}_\mathbf{C})}{n}+\frac{c'}{\sqrt{n}}
\end{align*}
Then, combining the above inequality with Theorem \ref{the-main-restult}, which can prove the result.
\end{proof}

\begin{remark}[Fast Rates]
    There are some results suggesting that the convergence rate of kernel $k$-means can be improved to $\mathcal{O}(k/n)$
    under certain assumptions on the distribution.
    \cite{chou1994distortion} pointed out that, for continuous densities satisfying certain regularity properties,
    including the uniqueness of the optimal cluster centers,
    the suitably scaled difference between the optimal and empirically
    optimal centers has an asymptotically multidimensional normal distribution,
    the expected excess risk of rate $\mathcal{O}(k/n)$.
    Further results were obtained by \cite{antos2005individual}, who proved
    that, for any fixed distribution supported on a given finite set,
    the convergence rate is $\mathcal{O}(k/n)$.
    They provided for more general
    (finite-dimensional) distribution conditions implying a rate of $\mathcal{O}(k\log n/n)$.
    \cite{levrard2015nonasymptotic} further showed that, if the distribution satisfies a margin condition,
    the convergence rate can also reach $\mathcal{O}(k/n)$.
    Based on the notion of local Rademacher complexity,
    \cite{koltchinskii2006local}
    showed that whenever the optimal clustering centers are unique,
    and the distribution has bounded support,
    the expected excess risk converges to zero at a rate faster than $\mathcal{O}(k/\sqrt{n})$.
    As pointed out by the authors, these conditions are, in general, difficult to verify.
    Moreover, these expected excess risk bounds are linearly  dependent on $k$.
    In the future, we will consider studying whether it is possible to prove a bound
    of $\mathcal{O}(\sqrt{k}/n)$ under certain strict assumptions.
\end{remark}

\subsection{Further Results: Beyond ERM}
So far we provided guarantees for $\mathbf{C}_n$, that is the ERM in $\mathcal{H}^k$.
Note that to obtain the optimal ERM $\mathbf{C}_n$ is a NP-Hard problem in general \cite{aloise2009np}.
In the following, we will consider the risk bound for a general $\tilde{\mathbf C}_n$,
which only require its empirical squared norm criterion is not far from that of $\mathbf{C}_n$.
\begin{corollary}
  \label{the-three-approximate}
  If $\forall \mathbf x\in\mathcal{X}, \|\Phi_\mathbf{x}\|\leq 1$,
  and $\mathbb{E}\big[\mathcal{W}(\tilde{\mathbf{C}}_n, \mathbb{P}_n)-\mathcal{W}(\mathbf{C}_n, \mathbb{P}_n)\big]\leq \zeta$,
  then for any $\delta>0$, there exists a constant $c$,
  with probability  at least $1-\delta$, we have
  \begin{align*}
    \mathbb{E}[\mathcal{W}(\tilde{\mathbf C}_n,\mathbb{P})]
    -\mathcal{W}^\ast(\mathbb{P})
    \leq  c \sqrt{\frac{k}{n}}\log ^{\frac{3}{2}+\delta}
    \left(\frac{n}{\max_i \tilde{\mathcal{R}}_n(\mathcal{F}_{\mathbf C_i})}\right)+\zeta.
  \end{align*}
\end{corollary}
From the above Theorem, one can see that
if the discrepancy of the empirical squared norm criterion of $\tilde{\mathbf C}_n$ and $\mathbf C_n$ is small,
that is $\zeta\leq \mathcal{O}(\sqrt{{k}/{n}})$,
the risk bound of $\tilde{\mathbf C}_n$ is (nearly) optimal.
\begin{proof}[Proof of sketch of Theorem \ref{the-three-approximate}]
We first show that $\mathbb{E}[\mathcal{W}(\tilde{\mathbf C}_n,\mathbb{P})]
    -\mathcal{W}^\ast(\mathbb{P})$ can be split into four parts:
\begin{align*}
&\underbrace{\mathbb{E}\Big[\mathcal{W}(\tilde{\mathbf C}_n,\mathbb{P})-\mathcal{W}(\tilde{\mathbf C}_n,\mathbb{P}_n)\Big]}_{A_1}
 +\underbrace{\mathbb{E}\Big[\mathcal{W}(\tilde{\mathbf C}_n,\mathbb{P}_n)-\mathcal{W}(\mathbf C_n,\mathbb{P}_n)\Big]}_{A_2}\\
 &~~~+\underbrace{\mathbb{E}\Big[\mathcal{W}(\mathbf C_n,\mathbb{P}_n)-\mathcal{W}(\mathbf C_n,\mathbb{P})\Big]}_{A_3}+
 \underbrace{\mathbb{E}\Big[\mathcal{W}(\mathbf C_n,\mathbb{P})\Big]-\mathcal{W}^\ast(\mathbb{P})}_{A_4}.
\end{align*}
$A_2$ can be bounded by $\zeta$, and
$A_4$ can be bounded as $\tilde{\mathcal{O}}(\sqrt{k/n})$ by Theorem \ref{the-three}.
 Both $A_1$ and $A_3$ can be bounded as $\tilde{\mathcal{O}}(\sqrt{k/n})$ by the Clustering Rademacher complexity.
\end{proof}

Lloyd's algorithm \cite{lloyd1982least} is the most popular $k$-means algorithm,
when coupled with a careful k-means++ seeding \cite{Arthur2007fff},
a good approximate solution $\tilde{\mathbf C}_n$ can be obtained.
Recently, based on a simple combination of $k$-means++ sampling with a local search strategy, an improved method is proposed \cite{lattanzi2019better}.
It has been shown that the empirical squared norm criterion of $\tilde{\mathbf C}_n$ can be up to a constant factor from the optimal empirical solution.
\begin{lemma}[\cite{lattanzi2019better}]
\label{lem-gahgag}
  If $\mathbf{C}_n^\mathcal{A}$ returned by
  the improved $k$-means++ algorithm with a local search strategy from \cite{lattanzi2019better},
   then,
  \begin{align*}
  \mathbb{E}_{\mathcal{A}}[\mathcal{W}(\mathbf{C}_{n}^\mathcal{A},\mathbb{P}_n)]\leq \beta\cdot \mathcal{W}(\mathbf{C}_n,\mathbb{P}_n),
\end{align*}
where $\beta$ is a constant and $\mathcal{A}$ is the randomness deriving from the $k$-means++ initialization.
\end{lemma}
Please refer to \cite{lattanzi2019better} for details.
Note that we can use the algorithm of \cite{lattanzi2019better} for kernel $k$-means
by replacing the Euclidean distance $\|\mathbf x_i-\mathbf x_j\|^2$
with $\|\Phi_i-\Phi_j\|_\mathcal{H}^2=\kappa(\mathbf x_i,\mathbf x_i)-2\kappa(\mathbf x_i,\mathbf x_j)+\kappa(\mathbf x_j,\mathbf x_j)$.
In the following,  we derive a risk bound for $\mathcal{W}(\mathbf{C}_{n}^\mathcal{A}$.
\begin{corollary}
\label{cor-gajgagh}
If $\forall \mathbf x\in\mathcal{X}, \|\Phi_\mathbf{x}\|\leq 1$, then for any $\delta>0$,
with probability  at least $1-\delta$, we have
  \begin{align*}
    \mathbb{E}_\mathcal{D}\left[\mathbb{E}_{\mathcal{A}}[\mathcal{W}(\mathbf{C}^\mathcal{A}_n,\mathbb{P})]\right]
    \leq \tilde{\mathcal{O}}\left(\sqrt{{k}/{n}}+\mathcal{W}^\ast(\mathbb{P})\right).
  \end{align*}
\end{corollary}
\begin{proof}[Proof of Sketch of Corollary \ref{cor-gajgagh}]
We show that $\mathbb{E}\Big[\mathbb{E}_\mathcal{A}[\mathcal{W}(\mathbf C_n^\mathcal{A},\mathbb{P})]\Big]$ can be split into three parts:
  \begin{align*}
    \underbrace{\mathbb{E}\Big[ \mathbb{E}_\mathcal{A}[\mathcal{W}(\mathbf C_n^\mathcal{A},\mathbb{P})]-
    \mathbb{E}_\mathcal{A}[\mathcal{W}(\mathbf C_n^\mathcal{A},\mathbb{P}_n)]\Big]}_{A_1}+\beta\cdot \underbrace{\mathbb{E}\Big[\mathcal{W}(\mathbf C_n,\mathbb{P}_n)-\mathcal{W}(\mathbf C_n,\mathbb{P})\Big]}_{A_2}+
    \beta\cdot\underbrace{\mathbb{E}\Big[\mathcal{W}(\mathbf C_n,\mathbb{P})\Big]}_{A_3}.
  \end{align*}
  Note that both $A_1$ and $A_2$ can be bounded as $\frac{2}{n}\mathcal{R}(\mathcal{G}_\mathbf{C})$,
  and $A_3$ can be bounded by $\mathcal{W}^\ast(\mathbb{P})+\tilde{\mathcal{O}}(\sqrt{k/n})$ from Theorem \ref{the-three}.
\end{proof}


\section{Risk Analysis of the Nystr\"{o}m Kernel $k$-Means}
Kernel $k$-means is one of the most popular clustering methods.
However, it requires the computation of an $n\times n$ kernel matrix.
This renders it non-scalable to large datasets that contain more than a few tens of thousands of points.
In particular, simply constructing and
storing the kernel matrix $\mathbf K$ takes $O(n^2)$ time and space.

The Nystr\"{o}m method \cite{drineas2005nystrom} is a popular  method for approximating the kernel matrix.
The properties of Nystr\"{o}m approximations for
kernel $k$-means have recently been  studied \cite{chitta2011approximate,cohen2015dimensionality,oglic2017nystrom,calandriello2018statistical,wang2019scalable}.
However, most of these works focus on the computational aspect of the problem.
To the best of our knowledge, the only one result providing excess risk guarantee for the Nystr\"{o}m kernel $k$-means is
\cite{calandriello2018statistical}. However, the excess risk bound of \cite{calandriello2018statistical} is  linear dependence on $k$.
In the following, we will improve $k$ to $\sqrt{k}$.

\subsection{Nystr\"{o}m Kernel $k$-Means}
To derive the excess risk bound of Nystr\"{o}m kernel $k$-means,
we first briefly introduce some notations.
Given a dataset $\mathcal{D}=\{\Phi_i\}_{i=1}^n$,
we denote with $\mathcal{I}=\{\Phi_i\}_{i=1}^m$ a dictionary (i.e., subset) of $m$ points $\Phi_i$ from $\mathcal{D}$.
Let $\mathcal{H}_m$ be a liner span of $\mathcal{I}=\{\Phi_i\}_{i=1}^m$,
\begin{align*}
  \mathcal{H}_m=\mathrm{span}\left\{\sum_{i=1}^m\alpha_i\Phi_i, \alpha_i\in\mathbb{R}, \Phi_i\in\mathcal{I}\right\},
\end{align*}
and $\mathcal{H}^k=\otimes_{i=1}^k\mathcal{H}_m$ be its Cartesian product.
The Nystr\"{o}m kernel $k$-means, i.e., the approximate
kernel $k$-means over $\mathcal{H}_m^k$, can be written as \cite{calandriello2018statistical}:
\begin{align*}
  \mathbf C_{n,m}&=\argmin_{\mathbf C\in\mathcal{H}_m^k}\frac{1}{n}\sum_{i=1}^n\min_{j=1,\ldots,k}\left\|(\Phi_i-\mathbf c_j)\right\|^2.
\end{align*}
The Nystr\"{o}m kernel $k$-means can be done in $\mathcal{O}(nm)$ space and $\mathcal{O}(nmkt + nm^2)$ time using $t$ steps of
Lloyd's algorithm for $k$ clusters \cite{lloyd1982least}, please refer to \cite{calandriello2018statistical} for details.
\subsection{Excess Risk Bound of Nystr\"{o}m Kernel $k$-Means}
Denote with $\Xi=\mathrm{Tr}(\mathbf K^\mathrm{T}(\mathbf K+\mathbf I)^{-1})$ the so-called effective
dimension of $\mathbf K$ \cite{rudi2017generalization}.
\begin{theorem}
\label{the-nsty}
If $\forall \mathbf x\in\mathcal{X}, \|\Phi_\mathbf{x}\|\leq 1$,
and the size of a uniform sampling is $m\gtrsim \frac{\sqrt{n}\log(1/\delta)\min(k,\Xi)}{\sqrt{k}},$
then, with probability at least $1-\delta$,
  we have
    \begin{align*}
    \mathbb{E}[\mathcal{W}(\mathbf C_{n,m},\mathbb{P})]
    -\mathcal{W}^\ast(\mathbb{P})= \tilde{\mathcal{O}}\left(\sqrt{\frac{k}{n}}\right).
  \end{align*}
\end{theorem}
Note that
${\sqrt{n}\min(k,\Xi)}/{\sqrt{k}}\leq \sqrt{nk}.$
Thus, from a statistical point of view, Theorem \ref{the-nsty} shows that
when sampling $\tilde{\Omega}(\sqrt{nk})$ points\footnote{$\tilde{\Omega}$ hides logarithmic terms.},
the Nystr\"{o}m kernel $k$-means achieves the same excess risk as the exact one.
This result shows that we can improve the computational aspect of kernel $k$-means using Nystr\"{o}m
embedding, while maintaining optimal generalization guarantees.

\begin{remark}
  In \cite{calandriello2018statistical}, they have reported that if $m\geq \tilde{\Omega}(\sqrt{n})$,
   a excess risk bound of rate $\tilde{\mathcal{O}}(k/\sqrt{n})$ for Nystr\"{o}m kernel $k$-means can be obtained,
   which seems to be better than our $\tilde{\Omega}(\sqrt{nk})$.
   However, it should be noted that the risk bound in \cite{calandriello2018statistical} is linearly dependent on $k$,
   while ours is linearly dependent on $\sqrt{k}$.
   From the proof of Lemma \ref{lem-nyfirst}, if we want to obtain a risk of linear dependence on $k$,
   we only need $m\geq \Omega(\sqrt{n}\log(1/\delta)\min(k,\Xi)/{k})=\tilde{\Omega}(\sqrt{n}),$
   which is the same as \cite{calandriello2018statistical}.
\end{remark}

\begin{proof}[Proof of Sketch of Theorem \ref{the-nsty}]
   We split the $\mathbb{E}[\mathcal{W}(\mathbf C_{n,m},\mathbb{P})]-\mathcal{W}^\ast(\mathbb{P})$ into four parts:
  \begin{align*}
    &\underbrace{\mathbb{E}[\mathcal{W}(\mathbf C_{n,m},\mathbb{P})-\mathcal{W}(\mathbf C_{n,m},\mathbb{P}_n)]}_{A_1}+\underbrace{\mathbb{E}[\mathcal{W}(\mathbf C_{n,m},\mathbb{P}_n)-\mathcal{W}(\mathbf C_n,\mathbb{P}_n)]}_{A_2}
    \\&~~~+\underbrace{\mathbb{E}[\mathcal{W}(\mathbf C_n,\mathbb{P}_n)-\mathcal{W}(\mathbf C_n,\mathbb{P})]}_{A_3}+
    \underbrace{\mathbb{E}[\mathcal{W}(\mathbf C_{n},\mathbb{P})]-\mathcal{W}^\ast({\mathbb{P}})}_{A_4}.
  \end{align*}
  $A_4$ can be bounded by $\tilde{\mathcal{O}}(\sqrt{{k}/{n}})$ using Theorem \ref{the-three}.
Both $A_1$ and $A_3$ can be bounded as $\tilde{\mathcal{O}}(\sqrt{k/n})$ using the Rademacher complexity.
$A_2$ can be bounded as $\tilde{\mathcal{O}}(\sqrt{k/n})$ by combining Lemma 1 and Lemma 2 of \cite{calandriello2018statistical}.
\end{proof}

\subsection{Further Results: Reducing the Sampling Points}
In Theorem \ref{the-nsty}, we know that we need $\tilde{\Omega}(\sqrt{nk})$ sampling points to guarantee the nearly optimal rate.
In the following,
we will show that we can reduce the sampling points under a basic assumption of the eigenvalues of the kernel matrix.
\begin{corollary}
\label{cor-first}
  Let $\lambda_i$ be the $i$-th eigenvalues of the kernel matrix $\mathbf{K}$, $i=1,\ldots, n$,
  $\lambda_{i+1}\leq \lambda_{i}$.
  If $\forall \mathbf x\in\mathcal{X}, \|\Phi_\mathbf{x}\|\leq 1$,
  and the eigenvalues satisfy the assumption that
  $\exists \lambda >1, c>0: \lambda_i\leq ci^{-\alpha}$,
  and the size of a uniform sampling is $m\gtrsim \sqrt{n}\log(1/\delta),$
  then, with probability at least $1-\delta$, we have
    \begin{align*}
    \mathbb{E}[\mathcal{W}(\mathbf C_{n,m},\mathbb{P})]
    -\mathcal{W}^\ast(\mathbb{P})= \tilde{\mathcal{O}}\left(\sqrt{{k}/{n}}\right).
  \end{align*}
\end{corollary}
The assumption of algebraically decreasing eigenvalues
of kernel function is a common assumption, for example,
met for the popular finite rank kernels and  shift invariant kernel \cite{williamson2001generalization}.
\begin{proof}[Proof of Sketch of Corollary \ref{cor-first}]
For $\exists \lambda >1, c>0: \lambda_i\leq ci^{-\alpha}$, we prove that
\begin{align*}
  \Xi=\mathrm{Tr}(\mathbf K^\mathrm{T}(\mathbf K+\mathbf I)^{-1})=\sum_{i=1}^n\frac{\lambda_i}{\lambda_i+1}
  \leq \left(1+\frac{c}{\alpha-1}\right)\sqrt{k}\leq \mathcal{O}(\sqrt{k}).
\end{align*}
\end{proof}

\subsection{Further Results: Beyond ERM}
\begin{corollary}
  \label{cor-second}
  Under the same assumptions as Corollary \ref{cor-first},
  if $\mathbb{E}\big[\mathcal{W}(\tilde{\mathbf C}_{n,m},\mathbb{P}_n)-\mathcal{W}(\mathbf C_{n,m},\mathbb{P}_n)\big]\leq \zeta$,
  and the size of a uniform sampling is $m\gtrsim \sqrt{n}\log(1/\delta),$ then,
with probability at least $1-\delta$, we have
    \begin{align*}
    \mathbb{E}[\mathcal{W}(\tilde{\mathbf C}_{n,m},\mathbb{P})]
    -\mathcal{W}^\ast(\mathbb{P})= \tilde{\mathcal{O}}\left(\sqrt{{k}/{n}}+\zeta\right).
  \end{align*}
\end{corollary}
  Similar as the proof of Theorem \ref{the-three-approximate}.

If we adopt the improved $k$-kernel means++ sampling with a local search strategy \cite{lattanzi2019better} for Nystr\"{o}m kernel $k$-means,
we can obtain the following results:
\begin{corollary}
\label{cor-gajgahgahgagh}
  Under the same assumptions as Corollary \ref{cor-first},
  If the size of a uniform sampling is $m\gtrsim \sqrt{n}\log(1/\delta),$
  then with probability  at least $1-\delta$, we have
  \begin{align*}
    \mathbb{E}_\mathcal{D}\left[\mathbb{E}_{\mathcal{A}}[\mathcal{W}(\mathbf{C}^\mathcal{A}_{n,m},\mathbb{P})]\right]
    \leq \tilde{\mathcal{O}}\left(\sqrt{{k}/{n}}+\mathcal{W}^\ast(\mathbb{P})\right).
  \end{align*}
\end{corollary}
  Similar as the proof of Corollary \ref{the-three-approximate}.
%

\section{Conclusion}
In this paper, we derive the nearly optimal risk bounds for both kernel $k$-means and Nystr\"{o}m kernel $k$-means,
which fills the gap of the optimal risk bounds for kernel $k$-means.
There are still some open questions to be further studied:
(1) Whether it is  possible to prove a (nearly) optimal risk bounds for random feature-based kernel $k$-means?
(2) Whether it is  possible to prove a bound of $\mathcal{O}(\sqrt{k}/n)$ under certain strict assumptions for kernel $k$-means?
(3) Whether it is possible to extend our results to deep $k$-means methods?
(4) Whether it is possible to extend our results to other unsupervised learning?

%

\bibliographystyle{abbrv}
\bibliography{ICML2020}

\newpage

\section{Proofs}
All the proofs are given in this section.
\subsection{Proof of Theorem \ref{the-main-restult}}
The proof of Theorem \ref{the-main-restult} closely follows from \cite{foster2019ell,srebro2010smoothness}.
We extend the results of linear space in \cite{foster2019ell} to the non-linear RKHS.
Specifically, we first show that the function $\varphi(\bm \nu)=\min(\nu_1,\ldots,\nu_k)$
is 1-Lipschitz with respect to the $L_\infty$-norm.
  Then, we prove that clustering Rademacher complexity of 1-Lipschitz function with respect to the $L_\infty$-norm can be bound by
  the maximum Rademacher complexity of the restriction of the function class along each coordinate.
\begin{lemma}
\label{lem-one}
 $\forall \bm \nu\in\mathbb{R}^k$, the function $\varphi(\bm \nu)=\min(\nu_1,\ldots, \nu_k)$ is $1$-Lipschitz with respect to the $L_\infty$ norm, i.e.,
\begin{align}
  \forall \bm \nu,\bm \nu'\in\mathbb{R}^k, |\varphi(\bm \nu)-\varphi(\bm \nu')|\leq \|\bm \nu-\bm \nu'\|_\infty.
\end{align}
\end{lemma}
\begin{proof}
Without loss of generality, we assume that $\varphi(\bm \nu)\geq \varphi(\bm \nu')$.
  Let
  $
    j=\argmin_{i=1,\ldots,k} \nu'_i,
  $
  then from the definition of $\varphi$,
   we know that  $\varphi(\bm \nu')=\nu'_j$.
  Thus, we can obtain that
  \begin{align*}
    |\varphi(\bm \nu)-\varphi(\bm \nu')|&=\varphi(\bm \nu)-\nu'_j&\\
    &\leq  \nu_j-\nu'_j  &(\text{by the fact that $\varphi(\bm \nu)\leq \nu_j$})\\
    &\leq \|\bm \nu-\bm \nu'\|_\infty.&
  \end{align*}
\end{proof}

\begin{definition}[Covering Number]
For a real-valued function class $\mathcal{M}\subseteq\{m:\mathcal{X}\rightarrow \mathbb{R}\}$.
The empirical $L_2$ covering number $\mathcal{N}_2(\mathcal{M},\varepsilon, \mathcal{S})$
is the size of the smallest set of sequences $V\subseteq \mathbb{R}^n$ for which
\begin{align*}
  \forall m\in \mathcal{M}, \exists v\in V, \sqrt{\frac{1}{n}\sum_{t=1}^n(m(\mathbf x_t)-v_t)^2}\leq \varepsilon.
\end{align*}
The empirical $L_\infty$ covering number $\mathcal{N}_\infty(\mathcal{M},\varepsilon, \mathcal{S})$
is the size of the smallest set of sequences $V\subseteq \mathbb{R}^n$ for which
\begin{align*}
  \forall m\in \mathcal{M}, \exists v\in V, \max_{1\leq t\leq n}|m(\mathbf x_t)-v_t|\leq \varepsilon.
\end{align*}
\end{definition}
\begin{lemma}
\label{lem-ncninfy}
  For all $\varepsilon>0$ and $\mathcal{S}=\{\mathbf x_1,\ldots, \mathbf x_n\}$, we have
  \begin{align*}
    \log \mathcal{N}_2(\mathcal{G}_\mathbf{C},\varepsilon, \mathcal{S})\leq k
    \cdot \max_i \log \mathcal{N}_\infty(\mathcal{F}_{\mathbf C_i},\varepsilon, \mathcal{S}).
  \end{align*}
\end{lemma}
\begin{proof}
  We assume that $\bm v_1,\ldots, \bm v_n$ is a sequence, $\bm v_i\in \mathbb{R}^k$,  and $f_\mathbf{C}=(f_{\mathbf c_1},\ldots, f_{\mathbf c_k})\in \mathcal{F}_\mathbf{C}$.
  From the definition of $\varphi$, we know that
  \begin{align*}
   \varphi(\bm a)=\min_{1\leq i\leq n}a_i.
  \end{align*}
  Thus, we can obtain that
  \begin{align*}
    \left(\frac{1}{n}\left(\varphi(f(\mathbf x_t))-\varphi(\bm v_t)\right)^2\right)^{1/2}&\leq \max_{1\leq t\leq n} |\varphi(f(\mathbf x_t))-\varphi(\bm v_t)|&\\
    &\leq \max_{1\leq t\leq n}\|f(\mathbf x_t)-\bm v_t\|_\infty.&(\text{Lemma \ref{lem-one}}).
  \end{align*}
  Let $V_i, \ldots, V_k$ be the sets that witness the $\ell_\infty$
  covering numbers for $\mathcal{F}_{\mathbf C_1},\ldots, \mathcal{F}_{\mathbf C_k}$ at scale $\varepsilon$.
  From the above inequality, one can see that the cartesian product $V=V_1\otimes V_2\cdots \otimes V_k$ can
  witness the $\ell_2$ covering number for $\varphi \circ \mathcal{F}$ at scale $\varepsilon$.
  Note that the size of the cartesian product $V$ at most $\max_i |V_i|^k$,
  thus, the size of sets that witnesses the $\ell_2$ covering number at most $\max_i |V_i|^k$.

%
%
\end{proof}
\begin{definition}[Fat-shattering dimension \cite{bartlett1998prediction}]
  $\mathcal{M}$ is said to shatter $\mathbf x_1,\ldots,\mathbf x_n$
  at scale $\gamma$ if there exists a sequence $v_1,\ldots,v_n$ such that
  \begin{align*}
    \forall \bm \epsilon \in \{\pm 1\}^n, \exists m\in\mathcal{M}, \text{~such that~} \epsilon_t\cdot(m(\mathbf x_t)-v_t)\geq \frac{\gamma}{2}, \forall t.
  \end{align*}
  The fat-shattering dimension $\mathrm{fat}_\gamma(\mathcal{M})$ is then defined as
  \begin{align*}
    \max
    \Big\{n \big| \exists \{\mathbf x_i\}_{i=1}^n\text{ such that}~\mathcal{M}~ \gamma\text{-shatter~} \{\mathbf x_i\}_{i=1}^n\Big\}.
  \end{align*}
\end{definition}
\begin{lemma}[\cite{rudelson2006combinatorics}, Theorem 4.4]
  \label{lem-two}
    For any $\delta\in (0,1)$, there exist constants
    $0<\tau<1$ and $C\geq 0$ such that for all $\varepsilon\in (0,1)$,
    \begin{align*}
      \log \mathcal{N}_\infty(\mathcal{F},\varepsilon, \mathcal{S})\leq C \eta_i\log\left(\frac{ne}{\eta_i\varepsilon}\right)\log^\delta\left(\frac{ne}{\eta_i}\right),
    \end{align*}
    where $\eta_i=\mathrm{fat}_{\tau\varepsilon}(\mathcal{F})$.
\end{lemma}

\begin{lemma}[\cite{srebro2010smoothness}, Lemma A.2]
For all $\varepsilon\geq \frac{2}{\tau n}\tilde{\mathcal{R}}_n(\mathcal{F})$,
it holds that
  \label{lem-three}
  \begin{align*}
   \mathrm{fat}_{\varepsilon}(\mathcal{F})\leq \frac{8}{n}\cdot\left(\frac{\tilde{\mathcal{R}}_n(\mathcal{F})}{\varepsilon}\right)^2,
   \text{ and  }\mathrm{fat}_{\varepsilon}(\mathcal{F})\leq n.
  \end{align*}
\end{lemma}
\begin{proof}[Proof of Theorem \ref{the-main-restult}]
  From the Theorem 2 of \cite{srebro2010smoothness}),
  we know that
  \begin{align}
    \label{leq-six}
    \mathcal{R}_n(\mathcal{G}_\mathbf{C}) \leq \inf_{\gamma>0}
    \left\{
        4\gamma n+12\int_{\gamma}^1\sqrt{\frac{\log \mathcal{N}_2(\mathcal{G}_\mathbf{C},\varepsilon, \mathcal{S})}{n}}d\varepsilon
    \right\}.
  \end{align}
Define $\beta_i=\frac{8}{n}\left(\frac{\tilde{\mathcal{R}}_n(\mathcal{F}_i)}{\tau\varepsilon}\right)^2,$ where $\tau$ is as in Lemma \ref{lem-two}.
Applying Lemma \ref{lem-ncninfy}, \ref{lem-two}, \ref{lem-three}, for any
$
  \frac{2}{\tau n}\tilde{\mathcal{R}}_n(\mathcal{F}_{\mathbf C_i})\leq \varepsilon < 1,
$
we have
\begin{align}
\label{equ-fffc}
\begin{aligned}
  \log \mathcal{N}_2(\mathcal{G}_\mathbf{C},\varepsilon, \mathcal{S})&\leq k \cdot \max_i \log \mathcal{N}_\infty(\mathcal{F}_{\mathbf C_i},\varepsilon, \mathcal{S}) &\text{ (Lemma \ref{lem-ncninfy})}\\
  &\leq \max_i C k\eta_i\log\left(\frac{ne}{\eta_i\varepsilon}\right)\log^\delta\left(\frac{ne}{\eta_i}\right) &\text{ (Lemma \ref{lem-two})}\\
  &\leq \max_i C k \beta_i\log \left(\frac{e^{2+\delta}n}{\beta_i\varepsilon}\right)\log \left(\frac{e^{2+\delta}n}{\beta_i}\right).
\end{aligned}
\end{align}
The last inequality uses the fact that for any $a,b>0$, the function
 $t\mapsto t\log(a/t)\log^\delta(b/t)$
is non-decreasing as long as $a\geq b\geq e^{1+\delta}t$.
Note that $\beta_i=\frac{8}{n}\left(\frac{\tilde{\mathcal{R}}_n(\mathcal{F}_i)}{\tau\varepsilon}\right)^2,$ so
from \eqref{equ-fffc}, we can obtain that
\begin{align*}
\log \mathcal{N}_2(\mathcal{G}_\mathbf{C},\varepsilon, \mathcal{S})
  &\leq \max_i \frac{C_1 k}{n}\left(\frac{\tilde{\mathcal{R}}_n(\mathcal{F}_{\mathbf C_i})}{\varepsilon}\right)^2
  \log \left(\frac{e^{2+\delta}n^2\varepsilon^2}{\tilde{\mathcal{R}}^2(\mathcal{F}_{\mathbf C_i})}\right) \log^\delta\left(\frac{e^{2+\delta}n^2\varepsilon}{\tilde{\mathcal{R}}^2(\mathcal{F}_{\mathbf C_i})}\right)\\
  &\leq \max_i \frac{C_1k}{n}\left(\frac{\tilde{\mathcal{R}}_n(\mathcal{F}_{\mathbf C_i})}{\varepsilon}\right)^2
  \log^{(\delta+1)}\left(\frac{e^{2+\delta}n^2}{\tilde{\mathcal{R}}^2(\mathcal{F}_{\mathbf C_i})}\right) &\text{(since $\varepsilon\leq 1$)}\\
  &= 2\max_i \frac{C_1k}{n}\left(\frac{\tilde{{\mathcal{R}}}_n(\mathcal{F}_{\mathbf C_i})}{\varepsilon}\right)^2
  \log^{(\delta+1)}\left(\frac{e^{(1+\delta/2)}n}{\tilde{\mathcal{R}}(\mathcal{F}_{\mathbf C_i})}\right).\\
  &=\max_i \frac{C_2k}{n}\left(\frac{\tilde{{\mathcal{R}}}_n(\mathcal{F}_{\mathbf C_i})}{\varepsilon}\right)^2
  \log^{(\delta+1)}\left(\frac{n}{\tilde{\mathcal{R}}(\mathcal{F}_{\mathbf C_i})}\right),
\end{align*}
To apply this bound in Equation \eqref{leq-six},
we set $\gamma=\frac{2}{\tau n}\left(\max_i\tilde{\mathcal{R}}_n(\mathcal{F}_{\mathbf{C}_i})\right)$,
which gives
\begin{align*}
  {\mathcal{R}_n}(\mathcal{G}_\mathbf{C})&\leq \frac{8}{\tau}\left(\max_i\tilde{\mathcal{R}}_n(\mathcal{F}_{\mathbf{C}_i})\right)\\
    &~~~~~+C_3\sqrt{k}\log^{\frac{1+\delta}{2}}\left(\max_i\tilde{\mathcal{R}}_n(\mathcal{F}_{\mathbf{C}_i})\right)
  \left(\frac{n}{\max_i\tilde{\mathcal{R}}_n(\mathcal{F}_{\mathbf{C}_i})}\right)\int_{\frac{2}{cn}\left(\max_i\tilde{\mathcal{R}}_n(\mathcal{F}_{\mathbf{C}_i})\right)}^1\varepsilon^{-1}d\varepsilon\\
  &\leq C_4\sqrt{k}\left(\max_i\tilde{\mathcal{R}}_n(\mathcal{F}_{\mathbf{C}_i})\right)\log^{\frac{3}{2}+\frac{\delta}{2}}\left(\frac{n}{\max_i\tilde{\mathcal{R}}_n(\mathcal{F}_{\mathbf{C}_i})}\right).
\end{align*}
This proves the result.
\end{proof}

\subsection{Proof of Proposition \ref{propo-lowbound} }
We first prove that the maximum Rademacher complexity can be bounded by $3\sqrt{n}$.
  Then, follows \cite{foster2019ell}, we show that there exist a hypothesis function $\mathcal{F}_\mathbf{C}$
such that $\mathcal{R}_n(\mathcal{G}_\mathbf{C})\geq \sqrt{\frac{kn}{2}}.$
\begin{lemma}
\label{the-middle}
If $\forall \mathbf x\in\mathcal{X}, \|\Phi_\mathbf{x}\|\leq 1$,
then for all $\mathbf C\in\mathcal{H}^k$, we have
  \begin{align*}
    \max_i\tilde{{\mathcal{R}}}_n(\mathcal{F}_{\mathbf C_i})\leq 3\sqrt{n}.
  \end{align*}
\end{lemma}
\begin{proof}
  For all $\mathcal{S}\in\mathcal{X}^n$,
  $\mathbf C\in\mathcal{H}^k$, and $i\in\{1,\ldots, k\}$,
  we have
  \begin{align}
  \begin{aligned}
  \label{equ-fagaga}
{\mathcal{R}}_n(\mathcal{F}_{\mathbf C_i})&=\mathbb{E}_{\bm \sigma}\sup_{f_{\mathbf{C}_i}\in\mathcal{F}_{\mathbf C_i}}\sum_{j=1}^n\sigma_j f_{\mathbf C_i}(\mathbf x_j)\\
    &=\mathbb{E}_{\bm \sigma}\sup_{\mathbf c_i\in \mathcal{H}}\sum_{j=1}^n\sigma_j\|\Phi_j-\mathbf c_i\|^2\\
    &=\mathbb{E}_{\bm \sigma}\sup_{\mathbf c_i\in \mathcal{H}}\sum_{j=1}^n\sigma_j[-2\langle \Phi_j, \mathbf c_i\rangle+\|\mathbf c_i\|^2+\|\Phi_j\|^2]\\
    &=\mathbb{E}_{\bm \sigma}\sup_{\mathbf c_i\in \mathcal{H}}\sum_{j=1}^n\sigma_j[-2\langle \Phi_j, \mathbf c_i\rangle+\|\mathbf c_i\|^2]\\
    &\leq 2\mathbb{E}_{\bm \sigma}\sup_{\mathbf c_i\in \mathcal{H}}\sum_{j=1}^n\sigma_j[\langle \Phi_j, \mathbf c_i\rangle]+\mathbb{E}_{\bm \sigma}\sup_{\mathbf c_i\in \mathcal{H}}\sum_{j=1}^n\sigma_j\|\mathbf c_i\|^2.
  \end{aligned}
  \end{align}
  Since $\|\mathbf c_i\|\leq 1$,
  thus we can obtain that
  \begin{align}
  \begin{aligned}
  \label{equ-faajhah}
    \mathbb{E}_{\bm \sigma}\sup_{\mathbf c_i\in \mathcal{H}}\sum_{j=1}^n\sigma_j\|\mathbf c_i\|^2&\leq \mathbb{E}_{\bm \sigma}|\sum_{j=1}^n\sigma_j|\\
    &\leq \sqrt{n} &\text{(by the Cauchy-Schwarz inequality)}.
  \end{aligned}
  \end{align}
Note that
    \begin{align}
    \begin{aligned}
    \label{eq-fagagag}
    \mathbb{E}_{\bm \sigma}\sup_{\mathbf c_i\in \mathcal{H}}\sum_{j=1}^n\sigma_j\langle \Phi_j, \mathbf c_i\rangle&=\mathbb{E}_{\bm \sigma}\sup_{\mathbf c_i\in \mathcal{H}}\Big\langle\sum_{j=1}^n\sigma_j \Phi_j, \mathbf c_i\Big\rangle\\
    &\leq \mathbb{E}_{\bm \sigma}\Big\|\sum_{j=1}^n\sigma_j\Phi_j\Big\| &\text{ (by $\|\mathbf c_i\|\leq 1$)}\\
    &\leq \sqrt{\mathbb{E}_{\bm \sigma}\Big\|\sum_{j=1}^n\sigma_j\Phi_j\Big\|^2}
    \\&\leq \sqrt{\sum_{i=1}^n\|\Phi_i\|^2}
     &\text{(by Lemma 24 (a) with $p=2$ in \cite{Lei2019ff})}\\
    &\leq \sqrt{n} ~~~~\text{ (since $\|\Phi_i\|\leq 1$)}.
  \end{aligned}
  \end{align}
  Substituting \eqref{equ-faajhah} and \eqref{eq-fagagag} into \eqref{equ-fagaga},
  which proves the result.
\end{proof}
\begin{proof}[Proof of Proposition \ref{propo-lowbound}]
  Let $\mathbf C=(\sigma_1\cdot e_1,\ldots, \sigma_2\cdot e_k),$ where $e_i$ denote the $i$th standard basis function in $\mathcal{H}$,
  and $\bm \sigma\in\{\pm 1\}^k$ are Rademacher variables.
  We choose the hypothesis space
\begin{align*}
  \mathcal{F}_\mathbf{C}=\Big\{f_\mathbf{C}=(f_{\sigma_1\cdot e_1},\ldots, f_{\sigma_k\cdot e_k}),
   f_{\sigma_i\cdot e_i}(\mathbf x)=\|\Phi_i-\sigma_i\cdot e_i\|^2,\bm \sigma\in\{\pm 1\}^k\Big\},
\end{align*}
Assume that $n$ is divisible by $k$. We set $$\Phi_1,\ldots, \Phi_{n/k}=e_1, \Phi_{(n+1)/k},\ldots, \Phi_{2n/k}=e_2,$$
and so on, and let $i_t$ be such that $\Phi_t=e_{i_t}$.
Thus, the empirical Rademacher complexity can be written as as
\begin{align}
\label{eq-ffetg}
 \begin{aligned}
 \mathcal{R}_n(\mathcal{G}_\mathbf{C})&=\mathcal{R}_n(\varphi\circ\mathcal{F}_\mathbf{C})\\
  &=\mathbb{E}_{\bm\sigma'}\max_{\bm \sigma}\sum_{t=1}^n\sigma_t'\min_{1\leq i\leq n}\|\Phi_i-\sigma_i\cdot e_i\|^2\\
  &=\mathbb{E}_{\bm\sigma'}\max_{\bm \sigma}\sum_{t=1}^n \sigma_t'\min_{1\leq i\leq n}\left(2-2\langle \Phi_i,\sigma_i\cdot e_i\rangle\right)\\
  &=2\mathbb{E}_{\bm\sigma'}\max_{\bm \sigma}\sum_{t=1}^n \sigma_t'\max_{1\leq i\leq n}\langle \Phi_i,\sigma_i\cdot e_i\rangle\\
  &=2\mathbb{E}_{\bm\sigma'}\max_{\bm \sigma}\sum_{t=1}^n \sigma_t'\max\{\sigma_{i_t},0\}\\
  &=2k\cdot \mathbb{E}_{\bm\sigma'}\max_{\sigma\in\{\pm 1\}}\sum_{t=1}^{n/k} \sigma_t'\max\{\sigma,0\}.
\end{aligned}
\end{align}
Using the Khintchine inequality \cite{young1976best},
we can obtain that
\begin{align*}
  \mathbb{E}_{\bm\sigma'}\max_{\sigma\in\{\pm 1\}}\sum_{t=1}^{n/k} \sigma_t'\max\{\sigma,0\}
=\frac{1}{2}\mathbb{E}_{\bm\sigma'}\left|\sum_{t=1}^{n/k}\sigma_t'\right|\geq \sqrt{\frac{n}{8k}}.
\end{align*}
Substituting the above inequality into \eqref{eq-ffetg}, we can obtain that
\begin{align*}
  \mathcal{R}_n(\mathcal{G}_\mathbf{C})=2k\cdot \mathbb{E}_{\bm\sigma'}\max_{\sigma\in\{\pm 1\}}\sum_{t=1}^{n/k} \sigma_t'\max\{\sigma,0\}\geq \sqrt{\frac{k n}{2}}.
\end{align*}
On the other hand, from Lemma \ref{the-middle}, for each $i$, we know that
\begin{align*}
  \max_i \tilde{\mathcal{R}}_n(\mathcal{F}_{\mathbf{C}_i})\leq 3\sqrt{n}.
\end{align*}
This proves the result.
\end{proof}

\subsection{Proof of Theorem \ref{the-three} }
\begin{proof}[Proof of Theorem \ref{the-three}]
The starting point of our analysis is the following elementary inequality (see \cite{Devroye1996}, Ch.8):
\begin{align}
\begin{aligned}
\label{leq-onee}
 \mathbb{E}[\mathcal{W}(\mathbf C_n,\mathbb{P})]-\mathcal{W}^\ast(\mathbb{P})&=\mathbb{E}\Big[\big(\mathcal{W}(\mathbf C_n,\mathbb{P})-\mathcal{W}(\mathbf C_n,\mathbb{P}_n)\big)+
  \big(\mathcal{W}(\mathbf C_n,\mathbb{P}_n))-\mathcal{W}^\ast(\mathbb{P}\big)\Big]\\
  &\leq \mathbb{E}\sup_{\mathbf C\in\mathcal{H}^k}\big(\mathcal{W}(\mathbf C,\mathbb{P}_n)-\mathcal{W}(\mathbf C,\mathbb{P})\big)+\sup_{\mathbf C\in\mathcal{H}^k}\mathbb{E}\left[\mathcal{W}(\mathbf C,\mathbb{P})-\mathcal{W}(\mathbf C,\mathbb{P}_n)\right]\\
  &\leq 2\mathbb{E}\sup_{\mathbf C\in\mathcal{H}^k}\big(\mathcal{W}(\mathbf C,\mathbb{P}_n)-\mathcal{W}(\mathbf C,\mathbb{P})\big).
\end{aligned}
\end{align}
Let $\mathbf x'_1,\ldots, \mathbf x'_n$ be an independent copy of $\mathbf x_1,\ldots, \mathbf x_n$,
independent of the $\sigma_i$'s.
Then, by a standard symmetrization argument \cite{Bartlett2002ff}, we can write
\begin{align}
\begin{aligned}
\label{leq-twwe}
\mathbb{E}\sup_{\mathbf C\in \mathcal{H}^k}\big(\mathcal{W}(\mathbf C,\mathbb{P}_n)-\mathcal{W}(\mathbf C,\mathbb{P})\big)
 &\leq \mathbb{E}\sup_{g_\mathbf{C}\in\mathcal{G}_\mathbf{C}} \frac{1}{n}\sum_{i=1}^n\sigma_i\left[g_\mathbf{C}(\mathbf x)-g_\mathbf{C}(\mathbf x')\right]\\
  &\leq 2 \mathbb{E}\sup_{g_\mathbf{C}\in\mathcal{G}_\mathbf{C}} \frac{1}{n}\sum_{i=1}^n\sigma_i g_\mathbf{C}(\mathbf x)\\&=\frac{2}{n}\mathcal{R}(\mathcal{G}_\mathbf{C}).
  \end{aligned}
\end{align}
From  Lemma A.2 in \cite{Oneto2015},
  with probability $1-\delta$, we have
  \begin{align}
  \label{lem-diff}
    \mathcal{R}(\mathcal{G}_\mathbf{C})\leq {\mathcal{R}}_n(\mathcal{G}_\mathbf{C})+\sqrt{{2n\log\left(\frac{1}{\delta}\right)}}.
  \end{align}
Thus, we can obtain that
\begin{align*}
  &~~~~\mathbb{E}\left[\mathcal{W}(\mathbf C_n,\mathbb{P})\right]-\mathcal{W}^\ast(\mathbb{P})
  \\&\leq \frac{4}{n}{\mathcal{R}_n}(\mathcal{G}_\mathbf{C})+4\sqrt{\frac{2\log(1/\delta)}{n}}&\text{(by \eqref{leq-onee}, \eqref{leq-twwe} and \eqref{lem-diff})}\\
  &\leq \frac{c\sqrt{k}}{n}\max_i {\tilde{\mathcal{R}}_n}(\mathcal{F}_{\mathbf C_i})\log^{\frac{3}{2}+\delta}
  \left(\frac{n}{{\tilde{\mathcal{R}}_n}(\mathcal{F}_{\mathbf C_i})}\right)+4\sqrt{\frac{2\log(1/\delta)}{n}} &\text{(by Theorem \ref{the-main-restult})}\\
  &\leq \frac{c_1\sqrt{k}}{\sqrt{n}}\log^{\frac{3}{2}+\delta}\left(\frac{n}{\max_i{\tilde{\mathcal{R}}_n}(\mathcal{F}_{\mathbf C_i})}\right)+4\sqrt{\frac{2\log(1/\delta)}{n}} &\text{(by Lemma \ref{the-middle})}\\
  &\leq c_2 \sqrt{\frac{k}{n}}\log ^{\frac{3}{2}+\delta}\left(\frac{n}{\max_i \tilde{\mathcal{R}}_n(\mathcal{F}_{\mathbf C_i})}\right).
\end{align*}
This proves the result.
\end{proof}

\subsection{Proof of Corollary \ref{the-three-approximate}}
\begin{proof}
Note that
\begin{align*}
  &\mathbb{E}[\mathcal{W}(\tilde{\mathbf C}_n,\mathbb{P})]-\mathcal{W}^\ast(\mathbb{P})
 \\=&\mathbb{E}\Big[\mathcal{W}(\tilde{\mathbf C}_n,\mathbb{P})-\mathcal{W}(\tilde{\mathbf C}_n,\mathbb{P}_n)+
 \mathcal{W}(\tilde{\mathbf C}_n,\mathbb{P}_n)-\mathcal{W}(\mathbf C_n,\mathbb{P}_n)
 \\&~~~+\mathcal{W}(\mathbf C_n,\mathbb{P}_n)-
 \mathcal{W}(\mathbf C_n,\mathbb{P})+\mathcal{W}(\mathbf C_n,\mathbb{P})-\mathcal{W}^\ast(\mathbb{P})\Big]\\
 &\leq \underbrace{\mathbb{E}\Big[\mathcal{W}(\tilde{\mathbf C}_n,\mathbb{P})-\mathcal{W}(\tilde{\mathbf C}_n,\mathbb{P}_n)\Big]}_{A_1}
 +\underbrace{\mathbb{E}\Big[\mathcal{W}(\tilde{\mathbf C}_n,\mathbb{P}_n)-\mathcal{W}(\mathbf C_n,\mathbb{P}_n)\Big]}_{A_2}\\
 &~~~+\underbrace{\mathbb{E}\Big[\mathcal{W}(\mathbf C_n,\mathbb{P}_n)-\mathcal{W}(\mathbf C_n,\mathbb{P})\Big]}_{A_3}+
 \underbrace{\mathbb{E}\Big[\mathcal{W}(\mathbf C_n,\mathbb{P})\Big]-\mathcal{W}^\ast(\mathbb{P})}_{A_4}.
\end{align*}
Note that $A_2$ is bounded by $\zeta$,
$A_4$ can be obtained from Theorem \ref{the-three},
and $A_1$ and $A_3$ can be bounded by the Rademacher complexity:
\begin{align*}
  A_1 \leq \mathbb{E}\sup_{\mathbf C\in\mathcal{H}^k}\left(\mathcal{W}(\mathbf C,\mathbb{P}_n)-\mathcal{W}(\mathbf C,\mathbb{P})\right)\leq \frac{2}{n}\mathcal{R}(\mathcal{G}_\mathbf{C}),\\
  A_3 \leq \mathbb{E}\sup_{\mathbf C\in\mathcal{H}^k}\left(\mathcal{W}(\mathbf C,\mathbb{P}_n)-\mathcal{W}(\mathbf C,\mathbb{P})\right)\leq \frac{2}{n}\mathcal{R}(\mathcal{G}_\mathbf{C}).
\end{align*}
Thus, we can obtain that
\begin{align}
\label{kmeans-approxiamte-mid}
  \mathbb{E}[\mathcal{W}(\tilde{\mathbf C}_n,\mathbb{P})]-\mathcal{W}^\ast(\mathbb{P})\leq \frac{4}{n}\mathcal{R}(\mathcal{G}_{\mathbf{C}})+c \sqrt{\frac{k}{n}}\log ^{\frac{3}{2}+\delta}
    \left(\frac{n}{\max_i \tilde{\mathcal{R}}_n(\mathcal{F}_{\mathbf C_i})}\right)+\zeta.
\end{align}
From \eqref{lem-diff}, we know that
  \begin{align*}
    \mathcal{R}(\mathcal{G}_\mathbf{C})&\leq {\mathcal{R}}_n(\mathcal{G}_\mathbf{C})+\sqrt{{2n\log\left(\frac{1}{\delta}\right)}}\\
    &\leq  c\sqrt{k}\max_i \tilde{\mathcal{R}}_n(\mathcal{F}_{\mathbf C_i})\log ^{\frac{3}{2}+\delta}\left(\frac{n}{\max_i \tilde{\mathcal{R}}_n(\mathcal{F}_{\mathbf C_i})}\right)
    +\sqrt{{2n\log\left(\frac{1}{\delta}\right)}} & \text{(by Theorem \ref{the-main-restult})}\\
    &\leq 3c\sqrt{kn}\log ^{\frac{3}{2}+\delta}\left(\frac{n}{\max_i \tilde{\mathcal{R}}_n(\mathcal{F}_{\mathbf C_i})}\right)+\sqrt{{2n\log\left(\frac{1}{\delta}\right)}} &(\text{by Lemma \ref{the-middle}})
  \end{align*}
  Substituting the above inequality into \eqref{kmeans-approxiamte-mid}, which proves the result.
\end{proof}
\subsection{Proof of Corollary \ref{cor-gajgagh}}
\begin{proof}[Proof of Corollary \ref{cor-gajgagh}]
  Note that
  \begin{align*}
    \mathbb{E}\left[\mathbb{E}_\mathcal{A}[\mathcal{W}(\mathbf C_n^\mathcal{A},\mathbb{P})]\right]
    &=\mathbb{E}\Big[ \mathbb{E}_\mathcal{A}[\mathcal{W}(\mathbf C_n^\mathcal{A},\mathbb{P})]-
    \mathbb{E}_\mathcal{A}[\mathcal{W}(\mathbf C_n^\mathcal{A},\mathbb{P}_n)]
    +\mathbb{E}_\mathcal{A}[\mathcal{W}(\mathbf C_n^\mathcal{A},\mathbb{P}_n)]\Big]\\
    &\leq \mathbb{E}\Big[ \mathbb{E}_\mathcal{A}[\mathcal{W}(\mathbf C_n^\mathcal{A},\mathbb{P})]-
    \mathbb{E}_\mathcal{A}[\mathcal{W}(\mathbf C_n^\mathcal{A},\mathbb{P}_n)]\Big]+
     \mathbb{E}\Big[\mathbb{E}_\mathcal{A}[\mathcal{W}(\mathbf C_n^\mathcal{A},\mathbb{P}_n)]\Big].
  \end{align*}
  From Lemma \ref{lem-gahgag}, we can obtain that
  \begin{align*}
    &\mathbb{E}\Big[\mathbb{E}_\mathcal{A}[\mathcal{W}(\mathbf C_n^\mathcal{A},\mathbb{P}_n)]\Big]
    \leq \beta\cdot \mathbb{E}[\mathcal{W}(\mathbf C_n),\mathbb{P}_n]\\
    = &\beta\cdot \mathbb{E}\Big[\mathcal{W}(\mathbf C_n,\mathbb{P}_n)-\mathcal{W}(\mathbf C_n,\mathbb{P})+\mathcal{W}(\mathbf C_n,\mathbb{P})\Big]\\
    \leq &\beta\cdot \mathbb{E}\Big[\mathcal{W}(\mathbf C_n,\mathbb{P}_n)-\mathcal{W}(\mathbf C_n,\mathbb{P})\Big]+
    \beta\cdot\mathbb{E}\Big[\mathcal{W}(\mathbf C_n,\mathbb{P})\Big].
  \end{align*}
  Thus, we can obtain that
  \begin{align*}
    \mathbb{E}\Big[\mathbb{E}_\mathcal{A}[\mathcal{W}(\mathbf C_n^\mathcal{A},\mathbb{P})]\Big]&\leq
    \underbrace{\mathbb{E}\Big[ \mathbb{E}_\mathcal{A}[\mathcal{W}(\mathbf C_n^\mathcal{A},\mathbb{P})]-
    \mathbb{E}_\mathcal{A}[\mathcal{W}(\mathbf C_n^\mathcal{A},\mathbb{P}_n)]\Big]}_{A_1}
    \\&+\beta\cdot \underbrace{\mathbb{E}\Big[\mathcal{W}(\mathbf C_n,\mathbb{P}_n)-\mathcal{W}(\mathbf C_n,\mathbb{P})\Big]}_{A_2}+
    \beta\cdot\underbrace{\mathbb{E}\Big[\mathcal{W}(\mathbf C_n,\mathbb{P})\Big]}_{A_3}.
  \end{align*}
  Note that
  \begin{align*}
    A_1, A_2&\leq \mathbb{E}\sup_{\mathbf C\in \mathcal{H}^k}\big(\mathcal{W}(\mathbf C,\mathbb{P}_n)
    -\mathcal{W}(\mathbf C,\mathbb{P})\big) &\\
    &\leq \frac{2}{n}{\mathcal{R}}_n(\mathcal{G}_\mathbf{C})+\sqrt{\frac{8}{n}{\log\left(\frac{1}{\delta}\right)}}  &\text{(by \ref{leq-twwe} and \ref{lem-diff}})\\
     &\leq \frac{2c\sqrt{k}}{n}\max_i {\tilde{\mathcal{R}}_n}(\mathcal{F}_{\mathbf C_i})\log^{\frac{3}{2}+\delta}\left(\frac{n}{{\tilde{\mathcal{R}}_n}(\mathcal{F}_{\mathbf C_i})}\right)
     +\sqrt{\frac{8\log(1/\delta)}{n}}&\text{(by Theorem \ref{the-main-restult})}\\
     &\leq \frac{6\sqrt{k}}{\sqrt{n}}\log^{\frac{3}{2}+\delta}\left(\frac{n}{{\tilde{\mathcal{R}}_n}(\mathcal{F}_{\mathbf C_i})}\right)
     +\sqrt{\frac{8\log(1/\delta)}{n}}&\text{(by Lemma \ref{the-middle})}\\
     &\leq \tilde{\mathcal{O}}\left(\sqrt{\frac{k}{n}}\right)
  \end{align*}
  By Theorem \ref{the-three}, we can obtain that
  \begin{align*}
    \mathbb{E}[\mathcal{W}(\mathbf C_n,\mathbb{P})]
    \leq  \mathcal{W}^\ast(\mathbb{P})+c \sqrt{\frac{k}{n}}\log ^{\frac{3}{2}+\delta}
    \left(\frac{n}{\max_i \tilde{\mathcal{R}}_n(\mathcal{F}_{\mathbf C_i})}\right).
  \end{align*}
  Thus, we can get
  \begin{align*}
    \mathbb{E}\Big[\mathbb{E}_\mathcal{A}[\mathcal{W}(\mathbf C_n^\mathcal{A},\mathbb{P}_n)]\Big]&
    \leq  \left(\frac{n}{\max_i \tilde{\mathcal{R}}_n(\mathcal{F}_{\mathbf C_i})}+\mathcal{W}^\ast(\mathbb{P})\right).
  \end{align*}
\end{proof}
\subsection{Proof of Theorem \ref{the-nsty}}
\begin{lemma}
  \label{Cmnp-pn}
  With probability at least $1-\delta$,
  \begin{align*}
    \mathbb{E}\left[\mathcal{W}(\mathbf C_{n,m},\mathbb{P}_n)-\mathcal{W}(\mathbf C_{n,m},\mathbb{P})\right]
    =\tilde{\mathcal{O}}\left(\sqrt{\frac{k}{n}}\right).
  \end{align*}
\end{lemma}
\begin{proof}
  Note that
  \begin{align*}
    &~~~~~\mathbb{E}\left[\mathcal{W}(\mathbf C_{n,m},\mathbb{P}_n)-\mathcal{W}(\mathbf C_{n,m},\mathbb{P})\right]\\
    &\leq \mathbb{E}\sup_{\mathbf C\in\mathcal{H}^k}\left[\mathcal{W}(\mathbf C,\mathbb{P}_n)-\mathcal{W}(\mathbf C,\mathbb{P})\right]
    \leq \frac{2}{n}\mathcal{R}(\mathcal{G}_\mathbf{C})&\text{(by \eqref{leq-twwe})}\\
    &\leq \frac{2}{n}{\mathcal{R}_n}(\mathcal{G}_\mathbf{C})+2\sqrt{\frac{2\log(1/\delta)}{n}} &\text{(by \eqref{lem-diff})}\\
    &\leq \frac{c\sqrt{k}}{n}\max_i {\tilde{\mathcal{R}}_n}(\mathcal{F}_{\mathbf C_i})\log^{\frac{3}{2}+\delta}
    \left(\frac{n}{{\tilde{\mathcal{R}}_n}(\mathcal{F}_{\mathbf C_i})}\right)+\sqrt{\frac{8\log(1/\delta)}{n}} &\text{(by Theorem \ref{the-main-restult})}\\
  &\leq \frac{c\sqrt{k}}{\sqrt{n}}\log^{\frac{3}{2}+\delta}\left(\frac{n}{{\tilde{\mathcal{R}}_n}(\mathcal{F}_{\mathbf C_i})}\right)+\sqrt{\frac{8\log(1/\delta)}{n}}&\text{(by Lemma \ref{the-middle})}\\
  &\leq \tilde{\mathcal{O}}\left(\sqrt{\frac{k}{n}}\right).
  \end{align*}
  This proves the result.
\end{proof}
\begin{lemma}
  \label{lem-nyfirst}
  If constructing $\mathcal{I}$ by uniformly sampling $m\geq C\sqrt{n}\log(1/\delta)\min(k,\Xi)/\sqrt{k}$,
  then for all $\mathcal{S}\in\mathcal{X}^n$, with probability at least $1-\delta$,
  we have
  \begin{align*}
    \mathcal{W}(\mathbf C_{n,m},\mathbb{P}_n)-\mathcal{W}(\mathbf C_n,\mathbb{P}_n)\leq C\sqrt{\frac{k}{n}},
  \end{align*}
  where $\Xi=\mathrm{Tr}(\mathbf K_n(\mathbf K_n+\mathbf I_n)^{-1})$
  is the effective dimension of $\mathbf K_n$.
\end{lemma}
    \begin{proof}
      This can be directly proved by combining Lemma 1 and Lemma 2 of \cite{calandriello2018statistical} by setting $\varepsilon=1/2$.
    \end{proof}
\begin{proof}[Proof of Theorem \ref{the-nsty}]
  Given our dictionary $\mathcal{I}$, we decompose
  \begin{align*}
    \mathbb{E}[\mathcal{W}(\mathbf C_{n,m},\mathbb{P})]-\mathcal{W}^\ast(\mathbb{P})
    =\underbrace{\mathbb{E}[\mathcal{W}(\mathbf C_{n,m},\mathbb{P})-\mathcal{W}(\mathbf C_{n},\mathbb{P})]}_{A}+\underbrace{\mathbb{E}[\mathcal{W}(\mathbf C_{n},\mathbb{P})]-\mathcal{W}^\ast({\mathbb{P}})}_{B}.
  \end{align*}
  The second pair $B$ can be bounded by $\tilde{\mathcal{O}}(\sqrt{\frac{k}{n}})$ using Theorem \ref{the-three}.
We further split $A$ as
\begin{align*}
  &~~~~~\mathbb{E}[\mathcal{W}(\mathbf C_{n,m},\mathbb{P})-\mathcal{W}(\mathbf C_n,\mathbb{P})]
  \\&=\underbrace{\mathbb{E}[\mathcal{W}(\mathbf C_{n,m},\mathbb{P})-\mathcal{W}(\mathbf C_{n,m},\mathbb{P}_n)]}_{A_1}
  +\underbrace{\mathbb{E}[\mathcal{W}(\mathbf C_{n,m},\mathbb{P}_n)-\mathcal{W}(\mathbf C_n,\mathbb{P}_n)]}_{A_2}
  +\underbrace{\mathbb{E}[\mathcal{W}(\mathbf C_n,\mathbb{P}_n)-\mathcal{W}(\mathbf C_n,\mathbb{P})]}_{A_3}.
\end{align*}
The last line $A_3$ is negative,
as $\mathbf C_n$ is optimal w.r.t. $\mathcal{W}(\cdot,\mathbb{P}_n)$.
The first line $A_1$  and  second line $A_2$ can both be bounded as $\tilde{\mathcal{O}}(\sqrt{k/n})$ using Lemma \ref{Cmnp-pn}
and  Lemma \ref{lem-nyfirst}, respectively.
\end{proof}

\subsection{Proof of Corollary \ref{cor-first}}
\begin{proof}
From the definition of effective dimension of $\mathbf K$,
we know that
\begin{align*}
  \Xi=\mathrm{Tr}(\mathbf K^\mathrm{T}(\mathbf K+\mathbf I)^{-1})&=\sum_{i=1}^n\frac{\lambda_i}{\lambda_i+1}\leq \sum_{i=1}^{\lfloor\sqrt{k}\rfloor} 1+\sum_{i=\lfloor\sqrt{k}\rfloor+1}^n\lambda_i\\
  &\leq \sqrt{k}+\sum_{i=\lfloor\sqrt{k}\rfloor+1}^n\lambda_i\leq \sqrt{k}+\sum_{i=\lfloor\sqrt{k}\rfloor+1}^n ci^{-\alpha}\\
  &\leq  \sqrt{k}+c\int_{\sqrt{k}}^\infty x^{-\alpha} d x= \sqrt{k}+\frac{c}{\alpha-1}\sqrt{k}^{1-\alpha}\\&\leq (1+\frac{c}{\alpha-1})\sqrt{k}.
\end{align*}
Thus, we can obtain that
\begin{align*}
  \frac{\min(k,\Xi)}{\sqrt{k}}\leq \frac{\Xi}{\sqrt{k}}\leq 1+\frac{c}{\alpha-1}.
\end{align*}
Substituting the above inequality into Theorem \ref{the-nsty}, which can prove this result.
\end{proof}

\subsection{Proof of Corollary \ref{cor-second}}
\begin{proof}[Proof of Corollary \ref{cor-second}]
  Note that
\begin{align*}
  &\mathbb{E}[\mathcal{W}(\tilde{\mathbf C}_{m,n},\mathbb{P})]-\mathcal{W}^\ast(\mathbb{P})
 \\=&\mathbb{E}\Big[\mathcal{W}(\tilde{\mathbf C}_n,\mathbb{P})-\mathcal{W}(\tilde{\mathbf C}_{m,n},\mathbb{P}_n)+
 \mathcal{W}(\tilde{\mathbf C}_{m,n},\mathbb{P}_n)-\mathcal{W}(\mathbf C_{m,n},\mathbb{P}_n)
 \\&~~~+\mathcal{W}(\mathbf C_{m,n},\mathbb{P}_n)-
 \mathcal{W}(\mathbf C_{m,n},\mathbb{P})+\mathcal{W}(\mathbf C_{m,n},\mathbb{P})-\mathcal{W}^\ast(\mathbb{P})\Big]\\
 &\leq \underbrace{\mathbb{E}\Big[\mathcal{W}(\tilde{\mathbf C}_{m,n},\mathbb{P})-\mathcal{W}(\tilde{\mathbf C}_{m,n},\mathbb{P}_n)\Big]}_{A_1}
 +\underbrace{\mathbb{E}\Big[\mathcal{W}(\tilde{\mathbf C}_{m,n},\mathbb{P}_n)-\mathcal{W}(\mathbf C_{m,n},\mathbb{P}_n)\Big]}_{A_2}\\
 &~~~+\underbrace{\mathbb{E}\Big[\mathcal{W}(\mathbf C_{m,n},\mathbb{P}_n)-\mathcal{W}(\mathbf C_{m,n},\mathbb{P})\Big]}_{A_3}+
 \underbrace{\mathbb{E}\Big[\mathcal{W}(\mathbf C_{m,n},\mathbb{P})\Big]-\mathcal{W}^\ast(\mathbb{P})}_{A_4}.
\end{align*}
Note that $A_2$ is bounded by $\zeta$,
$A_4$ can be obtained from Corollary \ref{cor-first},
and $A_1$ and $A_3$ can be bounded by the Rademacher complexity:
\begin{align*}
  A_1 \leq \mathbb{E}\sup_{\mathbf C\in\mathcal{H}^k}\left(\mathcal{W}(\mathbf C,\mathbb{P}_n)-\mathcal{W}(\mathbf C,\mathbb{P})\right)\leq \frac{2}{n}\mathcal{R}(\mathcal{G}_\mathbf{C}),\\
  A_3 \leq \mathbb{E}\sup_{\mathbf C\in\mathcal{H}^k}\left(\mathcal{W}(\mathbf C,\mathbb{P}_n)-\mathcal{W}(\mathbf C,\mathbb{P})\right)\leq \frac{2}{n}\mathcal{R}(\mathcal{G}_\mathbf{C}).
\end{align*}
Thus, we can obtain that
\begin{align}
\label{kmeans-approxiamte-aa-mid}
  \mathbb{E}[\mathcal{W}(\tilde{\mathbf C}_n,\mathbb{P})]-\mathcal{W}^\ast(\mathbb{P})
  \leq \tilde{\mathcal{O}}\left(\frac{\mathcal{R}(\mathcal{G}_{\mathbf{C}})}{n}+\sqrt{\frac{k}{n}}+\zeta\right).
\end{align}
From \eqref{lem-diff}, we know that
  \begin{align*}
    \mathcal{R}(\mathcal{G}_\mathbf{C})&\leq {\mathcal{R}}_n(\mathcal{G}_\mathbf{C})+\sqrt{{2n\log\left(\frac{1}{\delta}\right)}}\\
    &\leq  c\sqrt{k}\max_i \tilde{\mathcal{R}}_n(\mathcal{F}_{\mathbf C_i})\log ^{\frac{3}{2}+\delta}\left(\frac{n}{\max_i \tilde{\mathcal{R}}_n(\mathcal{F}_{\mathbf C_i})}\right)
    +\sqrt{{2n\log\left(\frac{1}{\delta}\right)}} & \text{(by Theorem \ref{the-main-restult})}\\
    &\leq 3c\sqrt{kn}\log ^{\frac{3}{2}+\delta}\left(\frac{n}{\max_i \tilde{\mathcal{R}}_n(\mathcal{F}_{\mathbf C_i})}\right)+\sqrt{{2n\log\left(\frac{1}{\delta}\right)}} &(\text{by Lemma \ref{the-middle}})
  \end{align*}
  Substituting the above inequality into \eqref{kmeans-approxiamte-aa-mid}, which proves the result.
\end{proof}

\subsection{Proof of Corollary \ref{cor-gajgahgahgagh}}
\begin{proof}[Proof of Corollary \ref{cor-gajgahgahgagh}]
  Note that
  \begin{align*}
    \mathbb{E}\left[\mathbb{E}_\mathcal{A}[\mathcal{W}(\mathbf C_{n,m}^\mathcal{A},\mathbb{P})]\right]
    &=\mathbb{E}\Big[ \mathbb{E}_\mathcal{A}[\mathcal{W}(\mathbf C_{n,m}^\mathcal{A},\mathbb{P})]-
    \mathbb{E}_\mathcal{A}[\mathcal{W}(\mathbf C_{n,m}^\mathcal{A},\mathbb{P}_n)]
    +\mathbb{E}_\mathcal{A}[\mathcal{W}(\mathbf C_{n,m}^\mathcal{A},\mathbb{P}_n)]\Big]\\
    &\leq \mathbb{E}\Big[ \mathbb{E}_\mathcal{A}[\mathcal{W}(\mathbf C_n^\mathcal{A},\mathbb{P})]-
    \mathbb{E}_\mathcal{A}[\mathcal{W}(\mathbf C_{n,m}^\mathcal{A},\mathbb{P}_n)]\Big]+
     \mathbb{E}\Big[\mathbb{E}_\mathcal{A}[\mathcal{W}(\mathbf C_{n,m}^\mathcal{A},\mathbb{P}_n)]\Big].
  \end{align*}
  From Lemma \ref{lem-gahgag}, we can obtain that
  \begin{align*}
    &\mathbb{E}\Big[\mathbb{E}_\mathcal{A}[\mathcal{W}(\mathbf C_{n,m}^\mathcal{A},\mathbb{P}_n)]\Big]
    \leq \beta\cdot \mathbb{E}[\mathcal{W}(\mathbf C_{n,m}),\mathbb{P}_n]\\
    = &\beta\cdot \mathbb{E}\Big[\mathcal{W}(\mathbf C_{n,m},\mathbb{P}_n)-\mathcal{W}(\mathbf C_{n,m},\mathbb{P})+\mathcal{W}(\mathbf C_{n,m},\mathbb{P})\Big]\\
    \leq &\beta\cdot \mathbb{E}\Big[\mathcal{W}(\mathbf C_{n,m},\mathbb{P}_n)-\mathcal{W}(\mathbf C_{n,m},\mathbb{P})\Big]+
    \beta\cdot\mathbb{E}\Big[\mathcal{W}(\mathbf C_{n,m},\mathbb{P})\Big].
  \end{align*}
  Thus, we can obtain that
  \begin{align*}
    \mathbb{E}\Big[\mathbb{E}_\mathcal{A}[\mathcal{W}(\mathbf C_{n,m}^\mathcal{A},\mathbb{P})]\Big]&\leq
    \underbrace{\mathbb{E}\Big[ \mathbb{E}_\mathcal{A}[\mathcal{W}(\mathbf C_{n,m}^\mathcal{A},\mathbb{P})]-
    \mathbb{E}_\mathcal{A}[\mathcal{W}(\mathbf C_{n,m}^\mathcal{A},\mathbb{P}_{n,m})]\Big]}_{A_1}
    \\&+\beta\cdot \underbrace{\mathbb{E}\Big[\mathcal{W}(\mathbf C_{n,m},\mathbb{P}_{n,m})-\mathcal{W}(\mathbf C_{n,m},\mathbb{P})\Big]}_{A_2}+
    \beta\cdot\underbrace{\mathbb{E}\Big[\mathcal{W}(\mathbf C_{n,m},\mathbb{P})\Big]}_{A_3}.
  \end{align*}
  Note that
  \begin{align*}
    A_1, A_2&\leq \mathbb{E}\sup_{\mathbf C\in \mathcal{H}^k}\big(\mathcal{W}(\mathbf C,\mathbb{P}_n)
    -\mathcal{W}(\mathbf C,\mathbb{P})\big) &\\
    &\leq \frac{2}{n}{\mathcal{R}}_n(\mathcal{G}_\mathbf{C})+\sqrt{\frac{8}{n}{\log\left(\frac{1}{\delta}\right)}}  &\text{(by \ref{leq-twwe} and \ref{lem-diff}})\\
     &\leq \frac{2c\sqrt{k}}{n}\max_i {\tilde{\mathcal{R}}_n}(\mathcal{F}_{\mathbf C_i})\log^{\frac{3}{2}+\delta}\left(\frac{n}{{\tilde{\mathcal{R}}_n}(\mathcal{F}_{\mathbf C_i})}\right)
     +\sqrt{\frac{8\log(1/\delta)}{n}}&\text{(by Theorem \ref{the-main-restult})}\\
     &\leq \frac{6\sqrt{k}}{\sqrt{n}}\log^{\frac{3}{2}+\delta}\left(\frac{n}{{\tilde{\mathcal{R}}_n}(\mathcal{F}_{\mathbf C_i})}\right)
     +\sqrt{\frac{8\log(1/\delta)}{n}}&\text{(by Lemma \ref{the-middle})}\\
     &\leq \tilde{\mathcal{O}}\left(\sqrt{\frac{k}{n}}\right)
  \end{align*}
  By Theorem \ref{the-three}, we can obtain that
  \begin{align*}
    \mathbb{E}[\mathcal{W}(\mathbf C_n,\mathbb{P})]
    \leq  \mathcal{W}^\ast(\mathbb{P})+c \sqrt{\frac{k}{n}}\log ^{\frac{3}{2}+\delta}
    \left(\frac{n}{\max_i \tilde{\mathcal{R}}_n(\mathcal{F}_{\mathbf C_i})}\right).
  \end{align*}
  Thus, we can get
  \begin{align*}
    \mathbb{E}\Big[\mathbb{E}_\mathcal{A}[\mathcal{W}(\mathbf C_{n,m}^\mathcal{A},\mathbb{P}_n)]\Big]&
    \leq  \left(\frac{n}{\max_i \tilde{\mathcal{R}}_n(\mathcal{F}_{\mathbf C_i})}+\mathcal{W}^\ast(\mathbb{P})\right).
  \end{align*}
\end{proof}

\end{document}